\documentclass[journal]{IEEEtran}
\usepackage{ifpdf}

\usepackage{cite}

\ifCLASSINFOpdf
  \usepackage[pdftex]{graphicx}
\else
\fi
\usepackage{amsmath}
\usepackage{algorithmic}
\usepackage{amsfonts}

\usepackage{array}
\usepackage{amsthm}
\usepackage{makecell,multirow,diagbox}

\usepackage{color}
\usepackage{algorithm}
\usepackage{algorithmic}
\usepackage{bm}
\usepackage{mathrsfs}
\usepackage{amssymb}
\usepackage{amsmath}
\usepackage[switch]{lineno}

\usepackage{url}

\newtheorem{lem}{Lemma}[section]
\newtheorem{prop}{Proposition}

\begin{document}
%
\title{Cost-effective Object Detection: Active Sample Mining with Switchable Selection Criteria}
%
%
%

\author{Keze Wang,~
        Liang Lin,~
        Xiaopeng Yan,~
        Ziliang Chen,~
        Dongyu Zhang,~
        and Lei Zhang,~\IEEEmembership{Fellow,~IEEE}
\thanks{This work was supported in part by National high level talents special support plan (Ten Thousand Talents Program),  in part by Guangdong ``Climbing Program'' Special Funds under Grant pdjhb0010, in part by National Natural Science Foundation of China (NSFC) under Grant U1611461 and Grant 61702565, in part by Ministry of Public Security Science and Technology Police Foundation Project of No. 2016GABJC48, in part by Science and Technology Planning Project of Guangdong Province of No.2017B010116001, in part by Hong Kong RGC General Research Fund (PolyU 152135/16E), and in part by the Hong Kong Polytechnic University’s Joint Supervision Scheme with the Chinese Mainland, Taiwan and Macao Universities (Grant no. G-SB20). 

K. Wang, L. Lin, X. Yan, Z. Chen, and D. Zhang are with the School of Data and Computer Science, Sun Yat-sen University, Guangzhou, China and the Engineering Research Center for Advanced Computing Engineering Software of Ministry of Education, China. The corresponding author is Liang Lin (e-mail: kezewang@gmail.com; linliang@ieee.org; yanxp3@mail2.sysu.edu.cn; zhangdy27@mail.sysu.edu.cn).}
\thanks{K. Wang and L. Zhang are with the Department of Computing, The Hong Kong Polytechnic University, Hong Kong. (e-mail: cslzhang@comp.polyu.edu.hk)}
}

%
%

\markboth{IEEE Transactions on Neural Networks and Learning Systems, 2018.}%
{}
%



\maketitle


\begin{abstract}
Though quite challenging, leveraging large-scale unlabeled or partially labeled data in learning systems (e.g., model / classifier training) has attracted increasing attentions due to its fundamental importance. To address this problem, many active learning (AL) methods have been proposed that employ up-to-date detectors to retrieve representative minority samples according to predefined confidence or uncertainty thresholds. However, these AL methods cause the detectors to ignore the remaining majority samples (i.e., those with low uncertainty or high prediction confidence). In this work, by developing a principled active sample mining (ASM) framework, we demonstrate that cost-effectively mining samples from these unlabeled majority data is key to training more powerful object detectors while minimizing user effort. Specifically, our ASM framework involves a switchable sample selection mechanism for determining whether an unlabeled sample should be manually annotated via AL or automatically pseudo-labeled via a novel self-learning process. The proposed process can be compatible with mini-batch based training (i.e., using a batch of unlabeled or partially labeled data as a one-time input) for object detection. In this process, the detector, such as a deep neural network, is first applied to the unlabeled samples (i.e., object proposals) to estimate their labels and output the corresponding prediction confidences. Then, our ASM framework is used to select a number of samples and assign pseudo-labels to them. These labels are specific to each learning batch, based on the confidence levels and additional constraints introduced by the AL process, and will be discarded afterward. Then, these temporarily labeled samples are employed for network fine-tuning. In addition, a few samples with low-confidence predictions are selected and annotated via AL. Notably, our method is suitable for object categories that are not seen in the unlabeled data during the learning process. Extensive experiments on two public benchmarks (i.e., the PASCAL VOC 2007/2012 datasets) clearly demonstrate that our ASM framework can achieve performance comparable to that of alternative methods but with significantly fewer annotations.   
\end{abstract}



%
\IEEEpeerreviewmaketitle

 \section{Introduction}
Benefiting from the state-of-the-art performance of deep convolutional neural networks (CNNs)~\cite{alexnet12NIPS,googlenet,He_2016_CVPR} obtained, remarkable progress has been achieved in object detection, which is one of the key objectives in computer vision. Through generating the candidate region/proposal of objects from the input image, object detection is converted into a region classification task. Features are usually extracted from candidate object regions via CNNs, e.g., R-CNN~\cite{rcnn14CVPR}, and conventional SVM / softmax classifiers are then used for final detection. Recently, most efforts have involved the design of powerful network architectures, e.g., ResNet~\cite{He_2016_CVPR} and SSD~\cite{ssd16ECCV}, to improve feature learning and computation speed. However, the question of how to {incrementally leverage large-scale unlabeled data to improve} detection performance is also a quite crucial and long-standing problem in {the learning system built by neural networks}. To solve this problem, three remaining technical issues regarding the use of training samples must be overcome: 
\begin{itemize}
\item  Annotating the samples used to train object detectors is usually a labor-intensive task. In contrast to other visual recognition tasks (e.g., image classification or action recognition), {a satisfactory annotation should contain} both an object's category label and its bounding box; thus, {annotating objects within a given image} is extremely time-consuming. Developing approaches for the automatic annotation of unlabeled data is a critical step in reducing the manual annotation burden. 
\item The training samples with the highest potential for improving performance are rare and difficult to identify. As reported in~\cite{ohem2016cvpr}, existing detection benchmarks usually contain an overwhelming number of ``easy'' examples and a small number of ``hard'' ones (i.e., informative samples with various {illumination conditions}, deformations, occlusions and other intra-class variations) that contribute to more effective and efficient training. As discussed in~\cite{Wang_2017_CVPR}, because {the relative difficulty of training samples follows} a long-tailed distribution, ``hard'' examples are uncommon. Therefore, finding such informative samples is a sophisticated task.
\item {Certain} training samples (e.g., outliers or noisy samples) may negatively affect the final detection performance. As reported {with regard to} SPP-net~\cite{spp15PAMI}, FRCN~\cite{frcn}, SSD~\cite{ssd16ECCV}, and RFCN~\cite{rfcn16NIPS} {on the PASCAL VOC benchmarks, detection performance can be substantially improved after the exclusion of the training samples marked as ``difficult'' by annotators}. The reason {for this improvement} may be inseparably linked to misleading or incorrectly annotated samples. {Although} such filtering is quite challenging, a sophisticated method is expected to be able to automatically filter out these outliers/noisy samples or otherwise alleviate the effect of using them to train object detectors.
\end{itemize}

To address the aforementioned issues, we focus on {learning object detectors in a cost-effective manner, which leverages} sample mining techniques to incrementally improve object detection with minimal user effort. {Recently,} active learning (AL) approaches~\cite{lewis1994sequential} have been proposed to progressively select and annotate the most informative unlabeled samples in a dataset to facilitate model refinement through user interaction when necessary. These approaches inspired us to attempt to give the more labor- and computation-intensive tasks to computers, while assigning the less labor-intensive tasks and those that require intelligence to humans~\cite{ISed}. {Therefore, the sample selection criteria play an important role in conventional AL pipelines, and are typically defined in accordance with the classification uncertainty of samples.} Specifically, the minority of unlabeled samples with low prediction confidences (i.e., high uncertainties), together with other informative criteria such as diversity and density, are generally treated as good candidates for {model retraining}. Recently, several AL-based approaches~\cite{llal11CVPR,id17CSR} have been proposed for object detection in a semi-supervised or weakly supervised manner. However, these approaches usually ignore the fact that the remaining majority samples (e.g., those with low uncertainty and high confidence) are also valuable for improving the detection performance. Moreover, manual annotations of unlabeled data are often noisy due to ambiguities or misunderstandings among different annotators, especially for the object detection task. Adding samples with incorrectly annotated bounding boxes may also reduce the detection performance. Therefore, both reducing user effort by mining the remaining majority samples and ensuring the appropriate treatment of outliers and noisy samples should be considered to improve the accuracy and robustness of object detectors.

Given sufficient unlabeled data, we attempt to overcome the {limitations of AL methods discussed} above by investigating recently proposed techniques. Curriculum learning (CL)~\cite{curriculun_learning} and self-paced learning (SPL)~\cite{letf11CVPR, spcl} are two learning regimes that mimic human and animal learning processes, in which training gradually progresses from easy to complex samples, providing a natural and iterative way to exploit labeled data for robust learning. In {CL}, a predefined learning constraint (i.e., a curriculum or curricular constraint) is {employed} to {incrementally} include additional labeled samples during training. In SPL, a weighted loss is introduced on all labeled samples, which acts as a general regularizer over the sample weights. By sequentially optimizing the model while gradually controlling the learning pace via the SPL regularizer, labeled samples can be incrementally added into the training process in a self-paced manner. Inspired by these techniques, several approaches~\cite{ceal16tcsvt, aspl15TPAMI} have been developed to improve AL for image classification by introducing a so-called pseudo-labeling strategy, which is intended to automatically select unlabeled samples with high prediction confidence and iteratively assign pseudo-labels to them in a self-paced manner. 

However, it is quite difficult to apply this pseudo-labeling strategy directly in the object detection task for two reasons. On one hand, {since} high-confidence samples are selected only in accordance with an empirical pace parameter, {the pseudo-annotations generated by the imperfect classifiers usually contain errors due to ignoring any feedback or guidance from the AL algorithm.} Furthermore, because object detection encompasses not only object classification but also accurate object localization, retraining the model with these incorrect pseudo-annotations may reduce the detection performance. On the other hand, the pseudo-labeling strategy must calculate sample weights for all training samples; thus, it is unsuitable for mini-batch-based training, which is the {predominant method used in} object detection pipelines. In addition, because the publicly available object detection benchmarks contain only a limited number of object categories, external image data {will usually contain some objects belonging to undefined} categories (i.e., third-party classes). During training, object detectors may misclassify {objects in} these undefined categories. Hence, a small number of user interactions are necessary to guide the pseudo-labeling strategy and maintain the control of training. 

To incorporate the pseudo-labeling strategy into the AL process and overcome the limitations of both, we propose a principled AL framework that performs active sample mining ({ASM) with switchable selection criteria}~\footnote{The source code will be released at \url{http://www.sysu-hcp.net/asm/}} to incrementally train robust object detectors over unlabeled or partially labeled samples without being compromised by noisy samples and outliers. Our {ASM} framework includes a novel {self-learning} process to facilitate model fine-tuning in a reliable and robust fashion. Specifically, given the prediction confidences, we propose to provisionally assign pseudo-labels to high-confidence region proposals under {additional} constraints introduced by the AL {algorithm and then} to retrain the {detector using these temporary} pseudo-annotations. Considering that the current state-of-the-art object detection pipelines (e.g., FRCN~\cite{frcn} and RFCN~\cite{rfcn16NIPS}) are all fine-tuned via stochastic gradient descent, the proposed {self-learning} process provides an effective and efficient way to {simultaneously perform both the pseudo-labeling of high-confidence samples and the fine-tuning of features} within every mini-batch iteration. This ensures that the high-confidence samples are pseudo-labeled with increasing accuracy as the model performance improves during network fine-tuning. Moreover, we impose a one-vs-rest strategy for handling undefined object categories; i.e., they are predicted negatively by all detectors, including the background detector, {to suppress} model drift when mining external unlabeled data, which may include natural scenes with many previously unseen object categories. {In summary}, the proposed {self-learning} process benefits the AL in two ways: i) it  significantly {reduces} the number of user-annotated samples required to improve the detection performance by {virtue of its} mini-batch training and undefined category handling {capabilities}, and ii) it effectively suppresses the misleading effects of samples that are incorrectly annotated via AL, {thanks to} the compactness and consistency of the high-confidence majority samples in the feature space {for training}.

Our {ASM} framework is formulated as a concise optimization problem. Specifically, we impose two different sets of sample mining scheme functions: one for the high-confidence sample pseudo-labeling mode and the other for the low-confidence sample annotation mode. A selector function is further introduced to {selectively and} seamlessly determine which mode should be executed to update the labels of the unlabeled region proposals. In this way, the proposed {self-learning} process and the AL process can jointly collaborate with each other to perform sample mining. Moreover, {our self-learning} process also considers {the} guidance and feedback from the AL process, making it suitable for large-scale scenarios. In detail, we define two curricula: a self-learning curriculum (SLC) and an active learning curriculum (ALC). The SLC represents a group of unlabeled region proposals with high potential for {accurately} pseudo-labeling, whereas the ALC represents difficult but informative region proposals suitable for active user annotation. During the training phase, the SLC is used to gradually incorporate pseudo-labeled samples, from simple examples to more complex ones, into the data used for retraining. By contrast, the ALC is used to intermittently add annotated samples into the training data in a complex-to-simple manner. Thus, we regard SLC and ALC as dual curricula. As they are updated via AL, these dual curricula effectively guide {the mining of the unlabeled data} using two completely different learning schemata. {Fig.~\ref{fig:architecture} illustrates the main components of our {ASM} framework, including region proposal generation and prediction, the pseudo-labeling of high-confidence samples via the proposed self-learning process, and the annotation of low-confidence samples via AL.}

\begin{figure}[t]
\center
\includegraphics[width = 0.8 \columnwidth]{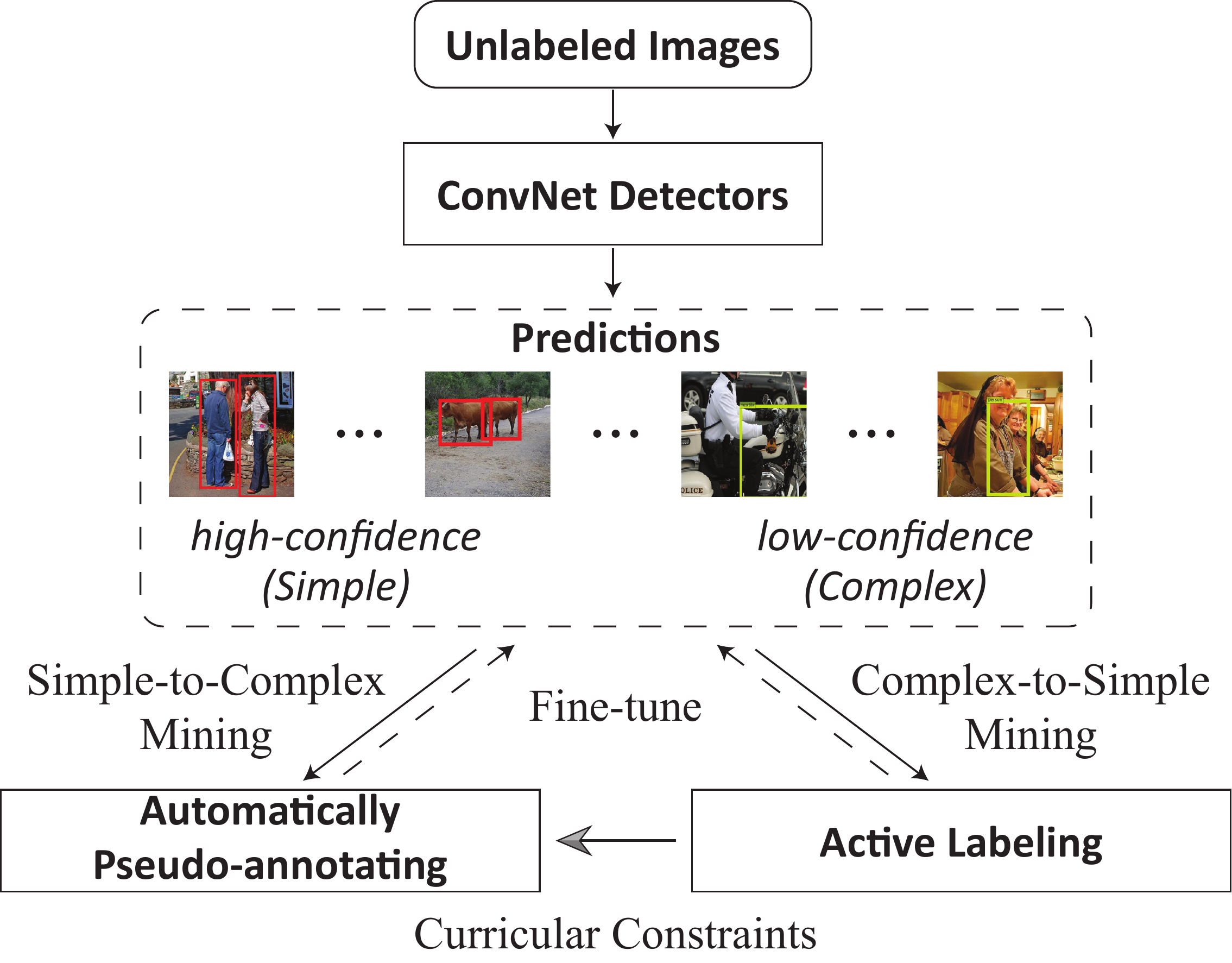}
\vspace{-10pt}
\caption{The proposed active sample mining (ASM) framework for object detection. Given the region proposals generated by the object detectors, our ASM framework includes modes for both the proposed automatically pseudo-annotating and active labeling. The arrows represent the workflow; the solid lines represent the data flow, and the dashed lines represent the fine-tuning operations. This figure shows that our ASM framework provides a rational approach for improving the detection of objects in unlabeled images by automatically distinguishing between high-confidence region predictions---which can be easily and faithfully recognized through simple-to-complex mining by computers via a self-learning process---and low-confidence ones, which can be discovered and labeled via complex-to-simple mining by interactively requesting user annotation. In addition, feedback from the active labeling process helps to guide the pseudo-labeling strategy via curricular constraints, which represent rules for discarding ambiguous predictions.
}
\vspace{-10pt}
\label{fig:architecture}
\end{figure}


The \textbf{main contributions} of this work are threefold. First, we present an {active sample mining} framework that can progressively boost the performance of object detectors by leveraging unlabeled or partially labeled data while minimizing the need for user annotation in a {cost-effective} way. Second, we develop an alternative optimization algorithm to facilitate the process of active sample mining with switchable selection criteria under the guidance of the proposed dual curricula. The developed algorithm can be scaled up for application to large-scale data by means of its ability to {selectively} and seamlessly switch between our self-learning process and the AL process for each unlabeled sample during mini-batch based training. {This is capable of effectively benefiting the application of network-based learning systems}. Third, our framework can adapt to unseen object categories in the unlabeled data during the progressive learning process. Extensive experiments on publicly available benchmarks (i.e., the PASCAL VOC 2007/2012 datasets) demonstrate that our framework not only can outperform the dominant state-of-the-art methods through {the active mining of} additional unlabeled data but also can achieve comparable performance with significantly fewer user annotations. 


The {remainder} of the paper is organized as follows. Section~\ref{sec:related_work} presents a brief review of related work. Section~\ref{sec:alg} provides an overview of the pipeline of our framework, followed by a discussion of the model formulation and optimization. {Experimental} results, comparisons and component analyses are presented in Section~\ref{sec:exper}, and Section~\ref{sec:conc} concludes the paper.

\section{Related Work}
\label{sec:related_work}
\textbf{Active Learning:} 
{Previous} work on AL {has mainly focused} on the sample selection strategy (i.e., selecting the most informative unlabeled samples for user annotation). Certainty-based selection~\cite{lewis1994sequential,tong2002support} is one of the most common {strategies} for AL. The certainty for {each new unlabeled sample} is measured according to the prediction confidence of the initial classifiers. Several SVM-based methods~\cite{tong2002support} {identify uncertain samples on the basis} of their distances to the decision boundary. {Recently, subspace learning~\cite{activesubspace,asl17tcyb} and frobenius-norm based representation~\cite{cb18tnnls} have also been applied to active learning, and obtained many achievements. Specifically, He {\em et al.}~\cite{activesubspace} proposed a novel active subspace learning algorithm which selects the most informative data points and uses them for learning an optimal subspace. Peng {\em et al.}~\cite{asl17tcyb} presented a called principal coefficients embedding method to automatically determine the optimal dimension of feature space and obtain the low-dimensional representation of a given data set under the unsupervised subspace learning scenario.} Regarding object detection via AL, Vijayanarasimhan {\em et al.}~\cite{llal11CVPR} proposed to refine part-based object detectors by actively requesting crowd-sourced annotations of images crawled from the Web. Rhee {\em et al.}~\cite{id17CSR} presented a {semi-supervised AL} method to improve object detection performance by leveraging the concept of diversity adopted from the AL paradigm. 

\textbf{Deep Learning for Object Detection:}
Because they benefit from learning feature representations directly from raw images, deep CNNs have achieved remarkable success in visual recognition, object detection and many other computer vision tasks~\cite{a1b,ob}. Recently, Krizhevsky {\em et al.}~\cite{alexnet12NIPS} designed a CNN that achieved a substantial improvement in image classification accuracy. In this state-of-the-art object detection method, R-CNN~\cite{rcnn14CVPR} is employed to extract features from category-independent region proposals of the input image. More recently, Shrivastava {\em et al.}~\cite{ohem2016cvpr} presented an online hard example mining algorithm to train region-based ConvNet detectors by eliminating several heuristics and hyperparameters. To harness rich information from the vast amount of visual data available, both semi-supervised and weakly supervised approaches for object detection have been proposed. Hoffman {\em et al.}~\cite{Hoffman_2015_CVPR} developed methods for training detectors that exploit joint training over both weak (image-level) and strong (bounding box) labels, and further transfer learned perceptual representations from strongly labeled auxiliary tasks. Yan {\em et al.}~\cite{em17} presented an expectation-maximization-based method for training object detectors on images with image-level labels in combination with some object-level-annotated images.


\textbf{Self-paced Learning:} Inspired by the principles of human/animal cognition, CL~\cite{curriculun_learning} was the first machine learning paradigm to adopt the concept of gradually adding samples to the training data in a controlled and meaningful (e.g., from easy to complex) sequential order called a {\em curriculum}. CL has been widely used to address a variety of computer vision problems, such as tracking~\cite{supancic2013self} and object detection~\cite{NIPS2012_4691}. Conventional CL approaches usually employ predefined sample weights to generate the training order of the samples. To jointly learn the sample weights and model parameters, Kumar {\em et al.}~\cite{spl_kumar} substantially advanced the learning philosophy of CL by proposing {a concise optimization paradigm named self-paced learning (SPL)} that includes a weighting scheme term on all samples and a general regularizer term over the sample weights. The weighting scheme enables training on easy to complex samples by assigning higher weights to samples with lower training losses. Recently, various other methods have also been developed using CL/SPL-related strategies~\cite{spl_reranking, spld, spcl}. Dong {\em et al.}~\cite{MSPLD} proposed an object detection framework that uses only a few bounding box labels per category by consistently alternating between detector amelioration and reliable sample selection. {Zhang {\em et al.}~\cite{bsd16IJCAI} proposed to bridge saliency detection to weakly-supervised object detection via the self-paced curriculum learning to gradually achieve faithful knowledge of multi-class objects from easy to hard. Wang {\em et al.}~\cite{sdm13tcyb} proposed to incorporate low-, mid- and high-level features into the the detection procedure via multiple-instance learning to overcome the challenges of inability and inconsistency for saliency detection. Wang {\em et al.}~\cite{vssc13tcsvt} further proposed to detect salient objects based on selective contrast, which intrinsically explores the most distinguishable component information in color, texture and location.}


\textbf{Self-learning:} In the literature, a few works~\cite{ceal16tcsvt, aspl15TPAMI, NIPS2011_4433, NIPS2012_4691, McClosky, MSPLD,ssm18CVPR} have attempted to leverage samples with high prediction confidence in the context of self-training. 
Chen {\em et al.}~\cite{NIPS2011_4433} proposed the slow addition of both target features and instances, among which the current model is the most confident, to the training set for domain adaption. Tang {\em et al.}~\cite{NIPS2012_4691} introduced a self-paced domain adaptation framework to adapt object detectors trained on images for application {in} videos. Wang {\em et al.}~\cite{ceal16tcsvt} proposed to employ a complementary sample selection strategy to progressively select the most informative samples and pseudo-label the samples with high prediction confidence for training. Wang {\em et al.}~\cite{aspl15TPAMI} further proposed an active SPL framework by incorporating the SPL technique into an AL pipeline via a concise active SPL optimization formulation.   


\section{Active Sample Mining with Switchable Selection Criteria}
\label{sec:alg}

In the context of object detection, suppose that $n$ object region proposals have been generated, corresponding to $m-1$ object categories and a background category. The training set $X=\{x_{i}\}_{i=1}^{n}$ contains all of these proposals as samples. Corresponding to the $m$ categories (including background), there are $m$ probabilistic detectors $\phi_j(x_i; \mathbf{W})$ for recognizing the category of each sample/proposal using the one-vs-rest strategy. Here, $\mathbf{W}$ denotes the shared parameter of our object detector network for all $m$ categories. Correspondingly, the label set of {the given sample} $x_i$ is denoted by $\mathbf{y}_i = \{y_i^{(j)}\}_{j=1}^m$, where $y_{i}^{(j)}$ is the label of {the} sample $x_i$ for the $j$-th object category (i.e., if $x_{i}$ is categorized as an instance of the $j$-th object category, then $y_{i}^{(j)}=1$; otherwise, $y_{i}^{(j)}=-1$). We present two important remarks on our problem setting: i) Only a small number (approximately 10\%) of the samples are initially annotated to obtain adequate initial models. Most of the sample labels $\mathbf{Y} = \{\mathbf{y}_i\}_{i=1}^n$ are unknown and must be determined in the subsequent learning phase. ii) The unlabeled or partially labeled data $\{\mathbf{x}_i\}_{i=1}^n$ are fed into the model in an incremental manner {to continuously boost the detector network.}

\subsection{Framework Formulation}
Under the premises {presented} above, our {ASM} framework is formulated as follows: 
		\begin{small}\begin{equation}
		\begin{gathered}
		\underset{\mathbf{W},\mathbf{Y}, \mathbf{V}}{\min} 
			\underset{\mathbf{U}}{\max} \   
			\mathbb{E}(\mathbf{U},\mathbf{V}; \mathbf{L}(X, \mathbf{Y}; \mathbf{W}), \gamma,{\bm \lambda}) \\= \sum_{i=1}^{n}\sum_{j=1}^{m} \max(u_i, v^{(j)}_i) l^{(j)}_i +  f_{SL}(\mathbf{v}_i, {\bm \lambda}) + f_{AL}(u_i,\gamma) \\\mathbf{s.t.} \ \forall \ i,j, \  y^{(j)}_i = \{-1,1\}, \mathbf{U} \in\mathbf{\Psi}^\gamma, \mathbf{V}\in\mathbf{\Psi}^{\bm \lambda}.
		\label{eq:obj}
		\end{gathered}
		\end{equation}\end{small}{{where} $\mathbf{L}(X, \mathbf{Y}; \mathbf{W}) = \{\mathcal{L}_i\}^n_{i=1} = \{\{l^{(j)}_i\}^m_{j=1}\}_{i=1}^n$ is the empirical loss set, {and} each $l^{(j)}_i$ is expressed as follows:
\begin{small}
	\begin{equation}
	\begin{aligned}
		l^{(j)}_i = &-\big(\frac{1 + y^{(j)}_i}{2}\log\phi_j (x_i;\mathbf{W}) +\frac{1-y^{(j)}_i}{2}\log(1-\phi_j(x_i;\mathbf{W})) \big).
		\label{eq:loss}
	\end{aligned}
	\end{equation} 
\end{small}
The objective function in Eqn.~(\ref{eq:obj}) can be elucidated as follows. Here, $f_{SL}(\cdot)$ and $f_{AL}(\cdot)$ denote the sample mining scheme functions for the pseudo-labeling of high-confidence samples via {our self-learning process (SL)} and the annotation of low-confidence samples via active learning (AL), respectively. Each sample $x_i$ has a latent indicator variable $u_i \in \{0, 1\}$ in the form of an annotation flag, and a latent weight variable $\mathbf{v}_i \in [0, 1)^m$ in the form of an $m$-dimensional weight vector. {The first latent weight variable set, $\mathbf{U}=\{u_i\}_{i=1}^n$, is employed to determine which samples should be annotated by active users.} The second latent weight variable set, $\mathbf{V}=\{\mathbf{v}_i\}_{i=1}^n=\{\{v_i^{(j)}\}^m_{j=1}\}_{i=1}^n$, is calculated from the class-specific loss of each sample according to the current detectors. The selector function $\max(u_i, v^{(j)}_i)$ is introduced to decide which mode will be executed to obtain $y_i^{(j)}$. Specifically, because $u_i\in\{0,1\}$ and $v_i^{(j)}\in[0,1)$, when $u_i=1$, it holds that $u_i > \max\{v_i^{(j)}\}^m_{j=1}$; thus, $\max(u_i,v_i^{(j)}) = u_i=1$. This indicates that the sample $x_i$ has been chosen by $f_{AL}(\cdot)$ for human annotation. When $u_i=0$, it holds that $u_i \leq \min\{v_i^{(j)}\}^m_{j=1}$; thus, $\forall j$, the selector function yields the weight $ \max(u_i,v_i^{(j)}) = v_i^{(j)}$ of $x_i$ for the classifier $\phi_j$, with a class-specific threshold $\lambda^{(j)}$.

\textbf{Dual Curricula}:
Due to the limited performance of the initial model, it is unfeasible to directly calculate each latent weight variable set $\mathbf{U}$ and $\mathbf{V}$ in its entirety for sample mining. Therefore, we leverage two curricula, {i.e.,} $\mathbf{\Psi}^{\gamma}$ and $\mathbf{\Psi}^{{\bm \lambda}}$, to selectively add unlabeled samples into the training set by constraining the optimization of $\mathbf{U}$ and $\mathbf{V}$. Reflecting human knowledge, these two curricula $\mathbf{\Psi}^{\gamma}$ and $\mathbf{\Psi}^{{\bm \lambda}}$ can provide perfect information to guide the model training. Specifically, these two curricula are initialized from $\{0, 1\}_{X}$ and $\{[0, 1)^m\}_X$, without the need for a predetermined learning order {as} \cite{spcl, liang16ijcai}. Suppose that in the $t$-th iteration, we have a set $A_{t-1}$ of samples that have been annotated by active users and a set $B_{t-1}$ of samples that have been discarded by active users (i.e., for being outside of the scope of the defined $m$ categories). The introduced curricula are then updated as shown below: 	

\begin{small}\begin{equation}
\begin{aligned}
\mathbf{\Psi}^{\gamma_t} = U_1^{\gamma_t}\times\cdots\times U^{\gamma_t}_n, \label{eq:c1}
\end{aligned}
\end{equation}\end{small}where $\forall i\in[n]$, if $x_i\in A_{t-1}$, then $U_i=\{1\}$; if $x_i\in B_{t-1}$, then $U_i=\{0\}$; and if $x_i\in X/ (A_{t-1}\bigcup B_{t-1})$, then $U_i=\{0,1\}$. Similarly,

\begin{small}\begin{equation}
\begin{aligned}
\mathbf{\Psi}^{{\bm \lambda}_t}= V^{\bm \lambda_t}_1\times\cdots\times V^{\bm \lambda_t}_n,\label{eq:c2}
\end{aligned}
\end{equation}\end{small}where $\forall i\in[n]$, if $x_i\in A_{t-1}\bigcup B_{t-1}$, then $V_i=\{0\}^m$, and if $x_i\in X/A_{t-1}\bigcup B_{t-1}$, then $V_i=[0,1)^m$. As shown in Eqn.~(\ref{eq:c1}), each dimension of $\mathbf{\Psi}^{\gamma_t}$ represents one of the samples in $X$. The weights of the samples in $A_t$ are all equal to $1$;
they are selected to fine-tune the network parameters based on the user annotations and to update $A_{t+1}$ for the next iteration. Meanwhile, the $u_i$ values of the remaining samples (i.e., $X/A_t$) are constrained to $\{0,1\}$. Note that the samples with $u_i=0$ are chosen by $f_{SL}(\cdot)$ for pseudo-labeling in accordance with the curriculum $\mathbf{\Psi}^{\bm \lambda_t}$ expressed in Eqn.~(\ref{eq:c2}). Each dimension of $\mathbf{\Psi}^{\bm \lambda_t}$ can be viewed as {a} $m$-dimensional vector that represents the scope of the weights of the class-specific loss expressed in Eqn.~(\ref{eq:loss}).} In summary, $\mathbf{\Psi}^{\bm \lambda_t}$ is used to progressively add pseudo-labeled samples, from simple examples to more complex ones, into the training data. By contrast, $\mathbf{\Psi}^{\gamma}$ is used to intermittently add annotated samples into the training data in a complex-to-simple manner. Therefore, we consider $\mathbf{\Psi}^{\bm \lambda_t}$ and $\mathbf{\Psi}^{\gamma}$ as dual curricula.

\textbf{Sample Mining Schemes:} Based on the {pre-knowledge discussed above}, we explain {how sample mining is performed} using $f_{SL}(\cdot)$ and $f_{AL}(\cdot)$. Motivated by the SPL technique~\cite{spcl}, we adopt the {following} linear scheme function $f_{SL}(\cdot)$:	

	\begin{small}
		\begin{equation}
		\label{eq:f1}
		\begin{aligned}
		f_{SL}(\mathbf{v}_i,{\bm \lambda}) &=  \overset{m}{\underset{j=1}{\sum}}\frac{1}{2}\lambda^{(j)}((v_i^{(j)})^2-2v_i^{(j)}) \\
		\mathbf{s.t.} \ \ \forall &j, \ \lambda^{(j)} > 0; \ v_i^{(j)} \in [0,1)\cap V^{\bm \lambda}_i,
		\end{aligned}
		\end{equation} 
	\end{small}where the threshold parameters ${\bm \lambda}=\{\lambda^{(j)}\}_{j=1}^m$ are used to define the high-confidence samples for each classifier $j$. Each $\lambda^{(j)}$ {is initially set to a small value}, {such that it is highly sensitive} to inaccurate pseudo-labels. {During processing, the values of the $\lambda^{(j)}$ are gradually increased} to allow more pseudo-labeled $x_i$ with larger losses for {network} fine-tuning.

For the minority of unlabeled proposals that are hard/informative, the low-confidence sample selection {is formulated using the following} scheme function $f_{AL}(\cdot)$:

\begin{small}
		\begin{equation}
		\label{eq:f2}
		\begin{aligned}
		&f_{AL}(u_i,\gamma) = -\gamma u_i \ \\ \mathbf{s.t.} \ \ \gamma &> 0, \forall u_i\in \mathbf{U}, u_i \in \{0,1\}\cap U^{\bm \gamma}_i, 
		\end{aligned}
		\end{equation}
\end{small}where the parameter $\gamma$ is a positive threshold that identifies the most informative samples for user annotation. The scheme function $f_{AL}(\cdot)$ {serves the opposite function to that of} the self-learning function $f_{SL}(\cdot)$. Specifically, {the objective of} $f_{AL}(\cdot)$ tends to {be maximized} by choosing hard samples with high uncertainty, {thereby driving} the framework to select the most informative samples for user annotation. {Then, active users} can annotate the samples that belong to the $m$ {known} classes and exclude those that lie outside this particular category set. {Notably, unlike} ${\bm \lambda}$, $\gamma$ {does not monotonically increase but instead may} change in various ways.

{We note that} the scheme function $f_{SL}(\cdot)$ represents a greedy self-learning strategy. It significantly {reduces} labor, but it is incapable of preventing the network from {being trained} on the semantic drift caused by accumulated prediction errors. Moreover, $f_{SL}(\cdot)$ also {strongly depends} on the initial parameter $\mathbf{W}$. {However}, the scheme function $f_{AL}(\cdot)$ {allows us to effectively overcome these weaknesses} of $f_{SL}(\cdot)$. Concretely, $f_{AL}(\cdot)$ selects samples for post-processing by active users. The user annotations obtained {from} $f_{AL}(\cdot)$ are considered reliable and {are} successively accepted until the training is completed. Counterintuitively, the pseudo-labels obtained {from} $f_{SL}(\cdot)$ are reliable only during their own training iteration and should be adaptively changed to more robustly guide the learning of the network parameters in each phase. Hence, just like the relationship between the dual curricula $\mathbf{\Psi}^{\gamma}$ and $\mathbf{\Psi}^{\bm \lambda}$, the scheme functions $f_{AL}(\cdot)$ and $f_{SL}(\cdot)$ also exert complementary influences on sample mining.

\subsection{Alternative Learning Strategy}
In our {ASM} framework, an alternating learning strategy is employed. {Specifically, the algorithm iteratively alternates between optimizing the parameters $\{ \mathbf{U}, \mathbf{V}, \mathbf{W}, \mathbf{Y}\}$ in accordance with the dual curricula $\{ \mathbf{\Psi}^{\gamma}$, $\mathbf{\Psi}^{\bm \lambda}\}$ and updating these curricula via AL in accordance with the dynamic $\gamma$ and $\lambda$ thresholds.}
In the following, we introduce the details of this optimization. The corresponding implementation of our {ASM} framework is discussed in Sect.~\ref{sec:implement}. 

\begin{figure*}[t]
\center
\includegraphics[width = 0.75 \textwidth]{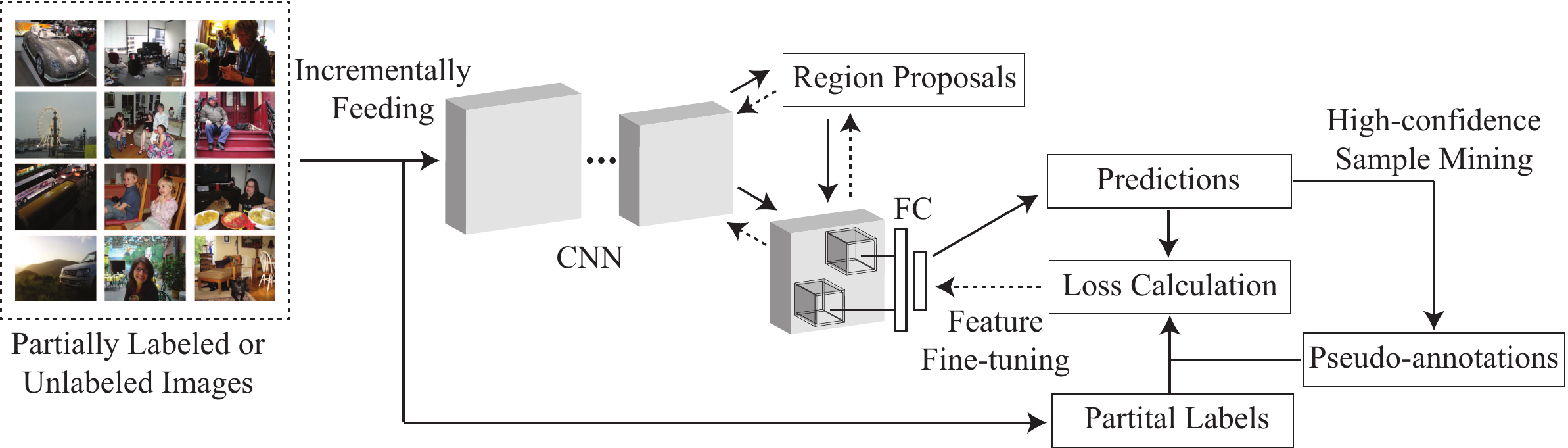}
\vspace{-10pt}
\caption{The workflow of the proposed high-confidence sample pseudo-labeling mode. The solid arrows represent forward propagation, and the dashed arrows represent backward propagation to fine-tune the network. Every mini-batch training iteration consists of three steps: i) region proposal generation and prediction from incrementally input partially labeled or unlabeled images, ii) pseudo-labeling based on high-confidence sample mining, and iii) feature fine-tuning by minimizing the loss between the predictions and training objectives (i.e., partial labels + pseudo-labels). Note that partial labels are absent when the images in a batch are all unlabeled.}
\vspace{-10pt}
\label{fig:highconfidence}
\end{figure*}

\subsubsection{\textbf{Initialization}}
The first step is to initialize the latent weight variable sets $\mathbf{V}_0$ and $\mathbf{U}_0$, the provided initial labels $\mathbf{Y}_0$, and the network parameter $\mathbf{W}_0$. Since our model starts from an unsupervised setting and needs to infer the object proposal categories, human-labeled samples serving as seeds are crucial for ensuring good ultimate performance. We choose these seeds randomly to populate $A_{0}$ and $B_{0}$. Following Eqn.~(\ref{eq:c1}) and Eqn.~(\ref{eq:c2}), we obtain the {initial} dual curricula $\mathbf{\Psi}^{\gamma_1}$ and $\mathbf{\Psi}^{\bm \lambda_1}$ before training. These curricula constrain $\mathbf{U}_0$ and $\mathbf{V}_0$. In practice, any initialization scheme for $\mathbf{U}_0\in\mathbf{\Psi}^{\gamma_1}$ and $\mathbf{V}_0\in\mathbf{\Psi}^{\bm \lambda_1}$ suffices. To obtain $\mathbf{Y}_0$, an active user annotates the seeds; then, the remaining samples in $X$ are pseudo-labeled as $\{-1\}^m_{j=1}$, {such that no} category to which they might belong {is indicated} during initialization. Finally, we choose a pre-trained classification model parameter (pre-trained on, e.g., ImageNet) and fine-tune it on the initial seed annotations to obtain $\mathbf{W}_0$.

\subsubsection{\textbf{Updating $\mathbf{U}$ and $\mathbf{V}$}}
{The purpose of this step is to provide an increasing number of training samples for use in each training iteration. Since the selector function in Eqn. (\ref{eq:obj}) acts on the loss weights stored in $\mathbf{U}$ and $\mathbf{V}$ to choose samples for pseudo-labeling and manual annotation, the values of $\mathbf{U}$ and $\mathbf{V}$ should be optimized appropriately.}

To this end, a max-min optimization of the latent weight variable sets $\mathbf{U}$ and $\mathbf{V}$ is performed, where $\mathbf{U}$ represents the utility of the samples based on uncertainty and $\mathbf{V}$ represents the reliability of the pseudo-labeling of the samples. Specifically, given \{$X, \mathbf{W}, \mathbf{Y}, \gamma, {\bm \lambda}$\}, the latent weight variable sets $\mathbf{V}$ and $\mathbf{U}$ are obtained by simplifying our {ASM} {formulation} in Eqn.~(\ref{eq:obj}) as follows:

\begin{small}	\begin{equation}
	\begin{aligned}
	\underset{ \mathbf{V}}{\min} \
	\underset{\mathbf{U}}{\max}& \  
	\mathbb{E}(\mathbf{U},\mathbf{V}; \mathbf{\overline{L}(X, \mathbf{Y}; \mathbf{W})}, \gamma,{\bm \lambda}) \\= \sum_{i=1}^{n}\sum_{j=1}^{m}\max(u_i,&v^{(j)}_i)\overline{l}^{(j)}_i + f_{SL}(\mathbf{v}_i,{\bm \lambda}) + f_{AL}(u_i,\gamma), \\
	\mathbf{U}&\in\mathbf{\Psi}^\gamma, \ \ \mathbf{V}\in\mathbf{\Psi}^\lambda,
	\label{eq:uv}
	\end{aligned}
	\end{equation}
\end{small}where $\mathbf{\overline{L}}(X, \mathbf{Y}; \mathbf{W})$ denotes that $\mathbf{L}(X, \mathbf{Y}; \mathbf{W})$ is fixed during the {updating} process. Next, we introduce two propositions to demonstrate how to optimize Eqn.~(\ref{eq:uv}); the proofs can be found in Appendix~\ref{sec:p1} and Appendix~\ref{sec:p2}, respectively.

{
\begin{prop}Given a training sample $x_i$, assume that 
	\begin{small}\begin{equation}\begin{gathered}
		E(u_i,\mathbf{v}_i;\mathcal{L}_i,\gamma, {\bm \lambda}) \\= \sum_{j=1}^{m}\max(u_i,v^{(j)}_i)l^{(j)}_i + f_{SL}(\mathbf{v}_i,{\bm \lambda})+ f_{AL}(u_i,\gamma).\label{eq5}
		\end{gathered}
		\end{equation}\end{small}
    Thus, we obtain $\{\mathbf{U}^\ast,\mathbf{V}^\ast\} = \overset{n}{\underset{i=1}{\bigcup}} \ \{u_i^\ast,\mathbf{v}^\ast_i\}$, where $\{\mathbf{U}^\ast,\mathbf{V}^\ast\}$ is the optimal solution to Eqn.~(\ref{eq:uv}) and $\{u^\ast_i,\mathbf{v}^\ast_i\}$ is the optimal solution to $\underset{u_i}{\max} \ \underset{\mathbf{v}_i}{\min} \  E(u_i,\mathbf{v}_i;\mathcal{L}_i,\gamma,{\bm \lambda})$.\end{prop} 
This proposition claims that Eqn.~(\ref{eq:uv}) can be decomposed to solve the instance-level min-max sub-problem $\underset{u_i}{\max} \ \underset{\mathbf{v}_i}{\min} \  E(u_i,\mathbf{v}_i;\mathcal{L}_i,\gamma,{\bm \lambda})$ for each sample $x_i$ in $X$: 
	\begin{small}\begin{equation}\begin{gathered}
		\underset{u_i}{\max} \ \underset{\mathbf{v_i}}{\min} \  E(u_i,\mathbf{v}_i;\mathcal{L}_{i},\gamma,{\bm \lambda}) \\= \sum_{j=1}^{m}\max(u_i,v_i^{(j)})l_i^{(j)} -\gamma u_i + \frac{1}{2}\overset{m}{\underset{j=1}{\sum}}\lambda^{(j)}((v_i^{(j)})^2-2v_i^{(j)}) \\ \mathbf{s.t.} \ \ \gamma > 0, u_i \in \{0,1\}\cap U_i, \\ \lambda^{(j)} > 0; \ \mathbf{v}_i=\{v_i^{(j)}\}^m_{j=1} \in [0,\epsilon]^m\cap V_i \subset [0,1)^m.
		\label{eq:uv2}\end{gathered}
		\end{equation}\end{small}Note that we introduce an adaptive threshold $\epsilon=\max_{x_i \in X} \{1 - \frac{l_i^{(j)}}{\lambda^{(j)}}\}$ to constrain the value of $v^{(j)}_i$. We leverage this threshold to create a marginal gap between the AL and SL processes to ensure more reliable sample selection results between the two complementary strategies. The behavior of this threshold is explained in detail below.  
		When certain conditions are satisfied, the solution to Eqn.~(\ref{eq:uv2}) can take the form of a threshold-based closed-form solution. More specifically, we provide a theoretical solution to the initial problem.
	\begin{prop}
		Given a sample $x_i\in X$, assume that we have a random initialization for $u^{(0)}_i\in\{0,1\}$ and $\mathbf{v}^{(0)}_i\in [0,1)^m$ and specific settings for $\{\gamma,{\bm \lambda},\overline{\mathcal{L}}_i\}$. Given $\epsilon\in(0,1)$, under the condition $\sum_{j=1}^{m}l_i^{(j)}\in(0,\gamma)\cup(\frac{\gamma}{1-\epsilon},\infty)$, the optimization of Eqn.~(\ref{eq:uv2}) converges to the following closed-form solution:
		\begin{small}
			\begin{equation}
			\label{eq:solution}
			\begin{aligned}
			\left\{
			\begin{aligned}
			u_i^\ast & = 1 ,\ (v_i^{(j)})^\ast = \epsilon \ \ \sum_{j=1}^{m}l_i^{(j)} > \frac{\gamma}{1-\epsilon}, \\
			u_i^\ast & = 0 ,\ (v_i^{(j)})^\ast = 0 \ \ \sum_{j=1}^{m}l_i^{(j)} < \gamma ,\ l_i^{(j)}>\lambda_j, \\
			u_i^\ast & = 0 ,\ (v_i^{(j)})^\ast = 1-\frac{l_i^{(j)}}{\lambda^{(j)}} \ \ \sum_{j=1}^{m}l_i^{(j)} < \gamma,\ \lambda_j(1-\epsilon)\leq l_i^{(j)}\leq\lambda_j,\\
			u_i^\ast & = 0 ,\ (v_i^{(j)})^\ast = \epsilon \ \ \sum_{j=1}^{m}l_i^{(j)} < \gamma ,\ l_i^{(j)}<\lambda_j(1-\epsilon).		
			\end{aligned}
			\right.
			\end{aligned}
			\end{equation}
		\end{small}
	\end{prop}In brief, this proposition claims that in the first iteration, AL will be triggered when $\sum_{j=1}^{m}l_i^{(j)} > \frac{\gamma}{1-\epsilon}$, whereas $\sum_{j=1}^{m}l_i^{(j)} < \gamma$ will lead to the SL process; these conditions depend on $\gamma$, $\epsilon$ and ${\bm \lambda}$. Obviously, the proposition cannot provide the values of $(u,v)$ when sample $i$ satisfies $\gamma\leq\sum_{j=1}^{m}l_i^{(j)}\leq\frac{\gamma}{1-\epsilon}$. The samples that lie within this margin will always show some level of class confusion in visualization, and they are not sufficiently important to deserve active annotation. Hence, we discard these training samples to achieve both more economical active selection and more robust pseudo-labeling. After the first iteration, an active user interacts with our model to augment $A_{t}$ and $B_{t}$. Then, we further update $\mathbf{\Psi}^\gamma$ and $\mathbf{\Psi}^{\bm \lambda}$ via Eqn.~(\ref{eq:c1}) and Eqn.~(\ref{eq:c2}), respectively. According to the definitions of $\mathbf{\Psi}^\gamma$ and $ \mathbf{\Psi}^{\bm \lambda}$ and the above analysis, when $x_i\in A_{t-1}$, we have $\forall t'>t$, $(u^{(t')\ast}_i, \mathbf{v}^{(t')\ast}_i)=\{1,\{0\}^m\}$. When $x_i\in B_{t-1}$, we have $\forall t'>t$, $(u^{(t')\ast}_i, \mathbf{v}^{(t')\ast}_i)=\{0,\{0\}^m\}$. Finally, when $x_i\in X/A_{t-1}\cup B_{t-1}$, we have $\{0,1\}\cap U_i = \{0,1\}$ and $[0,\epsilon]^m\cap V_i$. Thus, $\{u,v\}$ can be obtained via Proposition~(2).
	
	
}

\subsubsection{\textbf{Updating $\mathbf{Y}$}}
Holding $\mathbf{U}$ and $\mathbf{V}$ fixed as calculated above (denoted by $\mathbf{\overline{U}}$ and $\mathbf{\overline{V}}$), we update $\mathbf{Y}$ for the unlabeled region proposals. In our framework, the minority of these region proposals will be manually annotated by active users, {whereas} the majority will be pseudo-labeled {via our SL process. To complete this task, we develop the following two execution modes: a high-confidence sample pseudo-labeling {mode} and a low-confidence sample annotation mode. 

\textit{High-confidence Sample Pseudo-labeling Mode}: From the proposal set, we select and pseudo-label the high-confidence proposals for further model fine-tuning. As shown in Fig.~\ref{fig:highconfidence}, {in this mode}, {there are three steps in} each training iteration: i) generating region proposals for prediction from the incrementally input data, ii) predicting {classification results} for the generated proposals and {generating pseudo-labels} for use in fine-tuning via high-confidence sample mining, and iii) fine-tuning the network by minimizing the loss between the predictions and training objectives (i.e., partial labels + {pseudo-labels}). Because step i) is straightforward in existing region-based object detection pipelines, we focus on explaining steps ii) and iii), which include assigning pseudo-labels to temporarily update $\mathbf{Y}$ and the network parameter $\mathbf{W}$. Specifically, we perform high-confidence sample pseudo-labeling by optimizing $\mathbf{Y}$:

\begin{small}
\begin{equation}
\label{eq:ss}
		\begin{gathered}
		\underset{ \mathbf{Y}}{\min} \
		\  
		\mathbb{E}(\mathbf{\overline{U}},\mathbf{\overline{V}}; \mathbf{L}(X, \mathbf{Y};\mathbf{W}), \gamma, {\bm \lambda}) \\
		= \underset{\substack{x_i\in X\\\overline{v}^{(j)}_i\geq\overline{u}_i}}{\sum}\sum_{j=1}^{m}\max(\overline{u}_i,\overline{v}^{(j)}_i)l^{(j)}_i \\		
		\mathbf{s.t.} \ \ y^{(j)}_i = \{-1,1\},	\sum_{j=1}^{m} \vert y_{i}^{(j)} + 1 \vert \le 2,
		\end{gathered}
		\end{equation}
\end{small}where $v_i^{(j)}$ is fixed and can be treated as constant. We assign pseudo-labels only to $x_i$ that have a high probability of belonging to a certain object region (i.e., Eqn.~(\ref{eq:ss}) always has a clear solution). As indicated by the constraint $\sum_{j=1}^{m}|y^{(j)}_i+1|\leq 2$, our {ASM} framework largely excludes all samples for pseudo-labeling except under two conditions: i) when $y_{i}^{(j)}$ is predicted to be positive by one classifier but all {other classifiers produce negative predictions, or ii) when all} classifiers predict $y_{i}^{(j)}$ to be negative (i.e., $x_{i}$ is rejected by all classifiers and identified as {belonging to an} undefined object category). These are the rational cases for practical object detection in large-scale scenarios. {Note that we optimize $\mathbf{Y}$ by exhaustively attempting to assign -1 or 1 to each sample for all $m$ categories to minimize the loss function. The computational cost of this process is acceptable because we need to make only $m$+1 attempts under the constraint $\sum_{j=1}^{m}|y^{(j)}_i+1|\leq 2$.}


\textit{Support for an Undefined Object Category:} In this work, the problem of undefined object categories within external unlabeled data is considered by including an undefined sample set $B_t$. Inspired by the {one-vs-rest} strategy, we handle {all samples with undefined} object categories (those recognized as -1 by all classifiers including the background classifier, i.e., $\sum_j |y_i^{(j)} +1| =0$, where $y_i^{(j)}$ takes values in \{-1, 1\}) as belonging to a single undefined object category. This plays a crucial role in suppressing model drift when mining external unlabeled data, which may include many unseen object categories (e.g., the COCO benchmark has 60 more categories than the VOC 2007/2012 benchmark). 

\textit{Low-confidence Sample Annotation {Mode}}: 
After pseudo-labeling the high-confidence object proposals via the self-learning process, we employ data screening criteria using an {uncertainty-based strategy~\cite{lewis1994sequential,tong2002support}}. {The intent of using $f_{SL}(\cdot)$} is to flag most low-confidence unlabeled region proposals in $\mathbf{U}$ and then to have an active user annotate them in a category-by-category fashion as either positive or negative. Specifically, we utilize the current detector-based classifiers to predict the labels of the generated region proposals. In practice, as the detectors become better trained and more reliable, the samples with two or more labels predicted to be positive (i.e., recognized as belonging to more than one object category) will tend to be samples with weak illumination, large deformations, strong occlusions or other intra-class variations, i.e., the samples that cause ambiguity in the current classifiers. We thus consider these to be ``low-confidence" samples and require an active user to manually annotate them to boost the model performance. Other low-confidence criteria could also be utilized; however, we employ this simple strategy due to its intuitive rationality and efficiency. 

We require {an} active user to annotate the selected low-confidence samples in the $t$-th iteration of the AL process. After annotation, these samples {are} divided into two groups. One group contains samples that are outside the scope of the defined object categories (i.e., the undefined object category); these will be added to update $B_t$. The other group {contains} informative/hard samples and will be added to the training set to update $A_t$. We then employ the updated $A_t$ and $B_t$ to update the curriculum constraints $\mathbf{\Psi}^{\gamma}$ and $\mathbf{\Psi}^{\bm \lambda}$ via Eqn.~(\ref{eq:c1}) and Eqn.~(\ref{eq:c2}) by supplementing them with more {information} based on human knowledge.  

\subsubsection{\textbf{Updating $\mathbf{W}$}}
The network parameter $\mathbf{W}$ is updated based on $\{X,\mathbf{V}, \mathbf{U}, \mathbf{Y}, \mathbf{\Psi}^{\gamma}, \mathbf{ \Psi}^{\bm \lambda}\}$, where our {ASM} formulation given in Eqn.~(\ref{eq:obj}) degenerates to solving the following objective:
  \begin{small}
	\begin{equation}
	\label{eq:finetune}
	\begin{gathered}
	\underset{ \mathbf{W}}{\min} \
	\  
	\mathbb{E}(\mathbf{\overline{U}},\mathbf{\overline{V}}; \mathbf{L}(X, \mathbf{Y}; \mathbf{W}), \gamma, {\bm \lambda}) \\
	= \sum_{i=1}^{n}\sum_{j=1}^{m}\max(\overline{u}_i,\overline{v}^{(j)}_i)l^{(j)}_i \\
	=\underset{x_i\in A_t}{\sum}\sum_{j=1}^{m}l^{(j)}_i+\underset{\substack{x_i\in X/B_t\\\overline{v}^{(j)}_i\geq\overline{u}_i}}{\sum}\sum_{j=1}^{m}\overline{v}^{(j)}_il^{(j)}_i.
	\end{gathered}
	\end{equation}\end{small}In a deep learning scenario, this objective is decoupled into a set of mini-batches that can be readily solved by {efficient off-the-shelf} solvers using the stochastic gradient descent (SGD) approach. In our experiments, we employ the widely used standard SGD solver. We directly employ the annotated region proposals in $A_t$ for network fine-tuning, and we assign {temporary pseudo-labels} to the high-confidence region proposals (i.e., $\overline{v}^{(j)}_i\geq\overline{u}_i$) with sample weights of $\overline{v}^{(j)}_i$ for training.

\subsection{Implementation Details}
\label{sec:implement}
We initialize the model with the pre-trained feature representations from all region proposals $\{\mathbf{x}_{i}\}_{i=1}^{n}$ and {specify an initial set of} $m$ threshold parameters for the classifiers, ${\bm \lambda}=\{\lambda^{(j)}\}_{j=1}^{m}$. Next, we initialize the dual curricula $\mathbf{\Psi}^{\gamma}$ and $\mathbf{\Psi}^{\bm \lambda}$ {using the current user-annotated} samples $A_t$ and $B_t$ and the corresponding $\mathbf{Y}$, $\mathbf{V}$ and $\mathbf{U}$. In all our experiments, we empirically set ${\bm \lambda} = \{\lambda^{(j)}_0\}_{j=1}^m=\{ \lambda_0 \}_{j=1}^m=\{-\log0.9\}_{j=1}^m$ for each individual classifier, and we utilize a heuristic strategy for parameter updating. Specifically, for the $q$-th iteration, we compute the threshold parameters for Eqn.~(\ref{eq:obj}) as follows:

\begin{small}
\begin{equation}
\label{eq:lambda}
\lambda^{(j)}_q = \left\{
\begin{aligned}
&\lambda^{(j)}_{(q-1)} + \alpha * \eta^{(j)}_q, \ \ \ \ \ \ \text{ } 1 \le q \le  \tau, \\
&\lambda^{(j)}_{(q-1)}, \ \ \ \ \ \ \ \ \ \ \  \ \ \ \ \ \ \ \ \ \ \ \ q > \tau,
\end{aligned}
\right.
\end{equation}
\end{small}where $\eta^{(j)}_q$ is the negative logarithmic value of the average accuracy for the $j$-th classifier in the current iteration on the validation set and $\alpha$ is a parameter that controls the rate at which the threshold increases. Note that the threshold parameter ${\bm \lambda}$ should be eliminated after several updates to ensure the inclusion of a large number of unlabeled samples. Thus, we introduce an empirical threshold $\tau$ such that ${\bm \lambda}$ is updated only when $q \le \tau$. {Meanwhile, we empirically set a threshold parameter value of $\gamma$ = $0.5m$ for each unlabeled sample, where $m$ is the number of defined object categories.}

\begin{algorithm}[t]
\caption{{Active Sample Mining with Switchable Selection Criteria}}
\label{alg:alg_overview}
\begin{algorithmic}[1]
\REQUIRE Input dataset $\{\mathbf{x}_{i}\}_{i=1}^{n}$
\ENSURE Output model parameters $\{\mathbf{W}\}$
\STATE Initialize $\{\mathbf{x}_{i}\}_{i=1}^{n}$ with a pre-trained CNN, the curricula $\mathbf{ \Psi}^{\bm \lambda}$ and $\mathbf{\Psi}^{\gamma}$, $\{\mathbf{y}_{i}\}_{i=1}^{n}$, the latent weight variable sets $\mathbf{V}$ and $\mathbf{U}$, and the threshold parameters $\gamma$ and ${\bm \lambda}=\{\lambda^{(j)}_0\}_{j=1}^m$.
\STATE
\textbf{while} true \textbf{do} \\
\STATE \ \ \ \ \textbf{for all} mini-batches t = $1,...,T$ \textbf{do}
\STATE \ \ \ \ \ \ \ Update $\mathbf{W}$ via network fine-tuning using Eqn.~(\ref{eq:finetune});
\STATE \ \ \ \ \ \ \ Update $\mathbf{V}$ and $\mathbf{U}$ using Eqn.~(\ref{eq:solution});
\STATE \ \ \ \ \ \ \ Update $\{\mathbf{y}_i\}_{\mathbf{v}_i \ge u_i}$ in a self-learning manner using Eqn.~(\ref{eq:ss});
\STATE \ \ \ \ \textbf{end for}
\STATE \ \ \ \ Update the low-confidence sample sets $A_t$ and $B_t$;
\STATE \ \ \ \ \textbf{if} $A_t \cup B_t$ is not empty \textbf{do}
\STATE \ \ \ \ \ \ \ Update $\{\mathbf{y}_i\}_{i \in A_t}$ via AL;
\STATE \ \ \ \ \ \ \ Update $\mathbf{ \Psi}^{\bm \lambda}$ and $\mathbf{\Psi}^{\gamma}$ using Eqn.~(\ref{eq:c1}) and Eqn.~(\ref{eq:c2});
\STATE \ \ \ \ \textbf{else}
\STATE \ \ \ \ \ \ \ \textbf{break};
\STATE \ \ \ \ \textbf{end if}
\STATE \ \ \ \ Every $\beta$ iterations, update ${\bm \lambda}$ using Eqn.~(\ref{eq:lambda});
\STATE\textbf{end while}
\RETURN $\{\mathbf{W}\}$;
\end{algorithmic}
\end{algorithm}
The entire algorithm {is} summarized in Algorithm~\ref{alg:alg_overview}. It can be easily seen that this algorithm closely follows the pipeline of our {ASM} framework as depicted in Fig.~\ref{fig:architecture}. {Our {ASM} framework includes latent weight variable inference and unlabeled sample category handling; hence, it is impossible to explicitly provide a theoretical guarantee of convergence. However, it is clear that the convergence behavior is determined by two major factors: the capacity of our network and the training samples selected in the last stage, where $\bm \lambda$ and $\gamma$ define how many samples are ultimately incorporated into the data used for model training. Moreover, we have further imposed an adaptive threshold $\epsilon$ to create a marginal gap between the samples selected for the AL and SL processes. This marginal gap contains those unlabeled samples that are not selected for either the AL process or the SL process. Thus, our model can converge to a local optimum when all unlabeled samples are located within this marginal gap. In practice, our model does not require all unlabeled samples for training, and we can obtain a reasonably stable and reliable object detector within a certain maximum number of training iterations.}

\section{Experiments}
\label{sec:exper}

\subsection{Experimental Setup}
\textbf{Datasets and Parameter Settings:} To validate our {ASM} framework, we conducted experiments on the public PASCAL VOC 2007/2012 benchmarks~\cite{voc2007}, whose data are typically divided into the categories ``train'', ``val'' and ``test''. To evaluate the performance on these benchmarks {under the cost-effective object detection scenario, we define the following evaluation protocol:} we used {the data from} the VOC 2007 train and val sets as the initial annotated samples, whereas the VOC 2012 train and val sets were treated as unlabeled data. The active user annotation process consisted of fetching the annotations from the VOC 2012 train and val sets. Moreover, we used the object detection dataset COCO~\cite{coco14ECCV} as `secondary' unlabeled data. In other words, we performed sample mining on COCO only when all VOC 2012 train/val annotations had been used. {For evaluation}, we adopted the PASCAL challenge protocol: a correct detection should share more than 0.5 IoU with the ground-truth bounding box. The performances were evaluated using the mean average precision (mAP) metric. 

\begin{figure}[tbp]
\center
\includegraphics[width=0.7 \columnwidth]{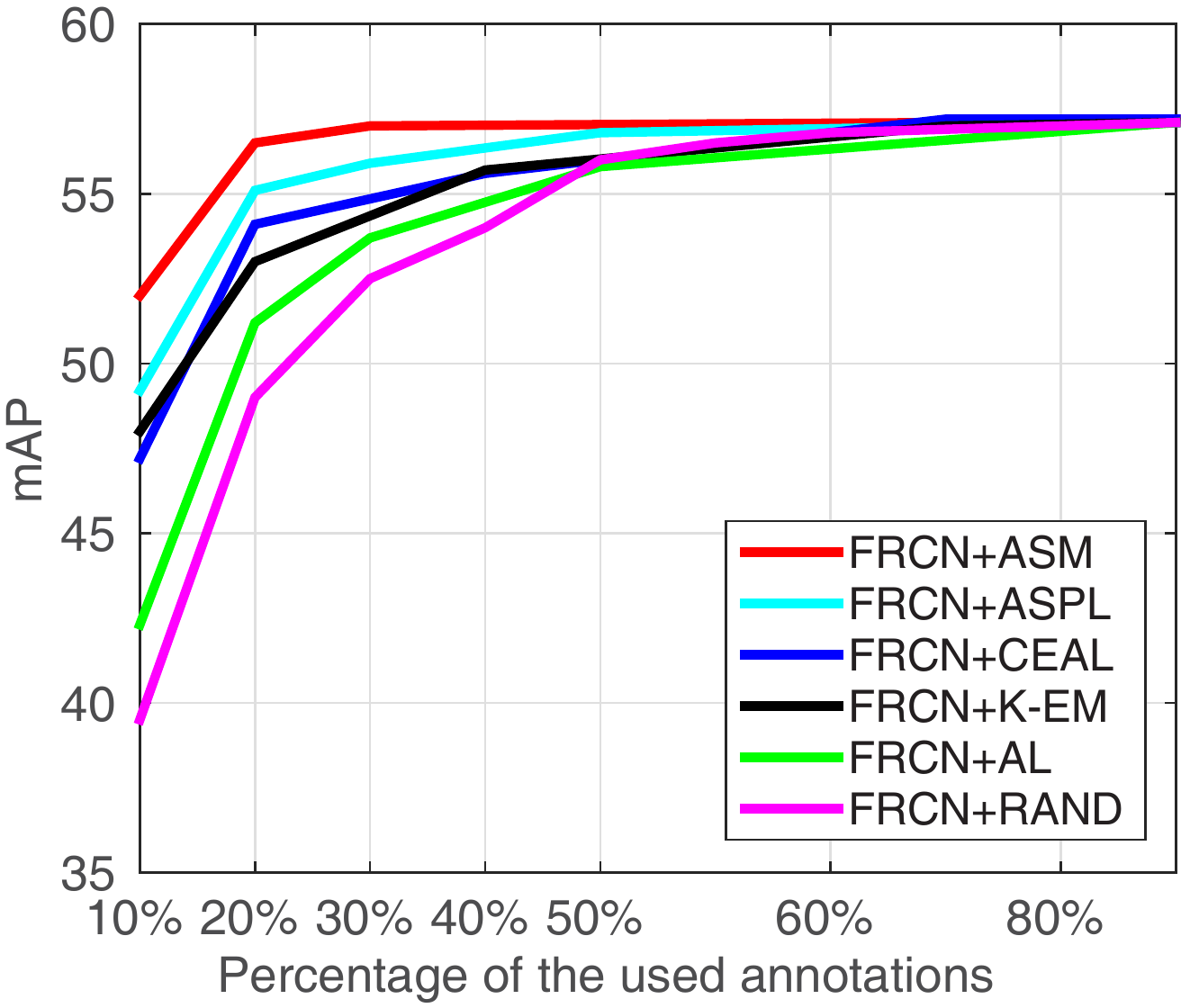}
\vspace{-10pt}
\caption{{Quantitative comparison of detection performance (mAP) on the VOC 2007 test set.}}\label{fig:cmp}
\vspace{-15pt}
\end{figure}

In all experiments, we trained $20$ object and background detectors on the VOC 2007/2012 benchmarks and set the parameters \{$\beta$, $k$, $\tau$, $\alpha$\} to \{10000, 50, 5, 0.08\}, respectively. The training strategy was the same as that described in~\cite{frcn} and \cite{rfcn16NIPS}. Specifically, the fine-tuning of the model, including the region proposal network (RPN), was performed using 4 GPUs (mini-batch size = 4) and a learning rate of 0.001 with a weight decay of 0.0005 and a momentum of 0.9. We employed multi-scale training in all the experiments. After shuffling labeled and unlabeled images at the beginning, our model randomly selects 4 images and {employs the RPN} to generate 300 unlabeled region proposals from each selected image. Then, our model adaptively decides to automatically pseudo-label or request active users to manually identify these proposals. Finally, the pseudo-labeled proposals and manual annotations are combined together to fine-tune our model (including RPN). In the testing phase, we follow the methods described in~\cite{spp15PAMI, rfcn16NIPS} to perform multi-scale inference. Note that the main difference in our training setting is that we treated the COCO train and val sets as unlabeled data for mining instead of pre-training the network. 

\begin{table}[tbp]
\footnotesize
\center
\setlength{\tabcolsep}{2pt}
\caption{Test set mAP results for VOC 2007/2012 obtained using the RFCN~\cite{rfcn16NIPS} pipeline. Annotation key: `initial` denotes the initial annotations, where `07' represents the annotations from the VOC 2007 train/val sets and `07+' represents the annotations from the VOC 2007 train/val/test sets; `annotated' denotes the percentage of appended object annotations from the VOC 2012 train/val sets relative to the number of initial annotations, while `pseudo' denotes the percentage of pseudo-labeled object proposals from the VOC 2012 train/val sets relative to the number of initial annotations.}\label{tab:07mAP12}\label{tab:12mAP07}\label{tab:ss}\label{tab:al}\label{tab:trivial}
\vspace{-5pt}
{\begin{tabular}{c|c| c c | c | c }
\hline
\hline
& Method & initial & test & annotated & mAP  \\
\hline
\multirow{13}*{(a)} & Faster RCNN & 07 & 07 & 100\% & 76.4 \\
& SSD513 & 07 & 07 & 100\%  & 80.6   \\ 
& RFCN & 07 & 07 & 0\% & 73.9 \\
& RFCN+RAND & 07 & 07 & 20\%  & 74.8 \\
& RFCN+RAND & 07 & 07 & 60\% & 76.5 \\
& RFCN+RAND & 07 & 07 & 100\% & 77.4  \\
& RFCN+RAND & 07 & 07 & 200\% & 79.8 \\
& RFCN+AL & 07 & 07 & 20\% & 77.6 \\
& RFCN+AL & 07 & 07 & 60\%  &78.4\\
& RFCN+AL & 07 & 07 & 100\%  &78.8  \\
& RFCN+{ASM} & 07 & 07 & 20\% & 78.1 \\
& RFCN+{ASM} & 07 & 07 & 60\% & 79.3  \\
& RFCN+{ASM} & 07 & 07 & 100\% & 79.6  \\
& RFCN+{ASM} & 07 & 07 & 200\% & \bf 81.8 \\
\hline
\multirow{12}*{(b)}& RFCN & 07+ & 12 & 0\% &  69.1   \\
& RFCN+RAND & 07+ & 12 & 10\% & 70.5 \\
& RFCN+RAND & 07+ & 12 & 30\% & 73.0  \\
& RFCN+RAND & 07+ & 12 & 50\% & 75.8  \\
& RFCN+RAND & 07+ & 12 & 100\% & 77.4  \\
& RFCN+AL & 07+ & 12 & 10\%  & 73.8  \\
& RFCN+AL & 07+ & 12  & 30\%  &  75.9  \\
& RFCN+AL& 07+ & 12  & 50\%  &76.2 \\
& RFCN+{ASM} & 07+ & 12 & 10\%  & 75.4  \\
& RFCN+{ASM} & 07+ & 12 & 30\% & 76.4  \\
& RFCN+{ASM} & 07+ & 12 & 50\% & 77.2  \\
& RFCN+{ASM} & 07+ & 12 & 100\% & \bf 78.3 \\

\hline
\hline
& Method & initial & test & pseudo & mAP  \\
\hline
\multirow{12}*{(c)} & RFCN & 07 & 07 & 0\%  & 73.9 \\
& RFCN+SL & 07 & 07 & 340\% &  77.5  \\
& RFCN+SL+AL & 07 & 07 & 460\% & 77.8  \\
& RFCN+AL+SL & 07 & 07 & 500\% & 78.2 \\
& RFCN+{ASM}~& 07 & 07 & 340\%  &  78.1 \\
& RFCN+{ASM}  & 07 & 07 & 400\% & \bf 79.3  \\
& RFCN & 07+ & 12 & 0\% & 69.1  \\
& RFCN+SL & 07+ & 12  & 130\%&  75.1 \\
& RFCN+SL+AL & 07+ & 12 & 230\% & 76.1  \\
& RFCN+AL+SL & 07+ & 12 & 300\% & 76.7  \\
& RFCN+{ASM}& 07+ & 12 & 130\% & 75.4 \\
& RFCN+{ASM} & 07+ & 12 & 190\% & \bf 77.2    \\
\hline
\hline
\end{tabular}}
\vspace{-10pt}
\end{table}

\textbf{Compared Approaches:} 
To validate the proposed {ASM} framework, we compared it with the CEAL~\cite{ceal16tcsvt} and K-EM~\cite{em17} approaches using the FRCN pipeline with AlexNet~\cite{alexnet12NIPS} (well pre-trained on ImageNet). Note that because the CEAL approach is designed for image classification (i.e, it is not mini-batch friendly), we extended it for object detection by alternately performing sample selection (i.e., informative sample annotation via the least confidence criterion and high-confidence sample pseudo-labeling) and CNN fine-tuning. We directly used the results of K-EM~\cite{em17} as reported in~\cite{em17}. We use the abbreviations ``FRCN+{ASM}'', ``FRCN+CEAL'', and ``FRCN+K-EM'', respectively, to denote these methods. We also included a baseline method ``FRCN+RAND'', in which region proposals are randomly selected for user annotations. To demonstrate that our {ASM} approach can be generalized to different network architectures and object recognition frameworks, we {also} incorporated our {ASM} method into the FRCN pipeline~\cite{frcn} with VGGNet~\cite{vgg} (well pre-trained on ImageNet) and the new state-of-the-art RFCN pipeline~\cite{rfcn16NIPS} with ResNet-101~\cite{He_2016_CVPR} (well pre-trained on ImageNet). We use the abbreviations ``FRCN+{ASM}'' and ``RFCN+{ASM}'', respectively, to denote these variants of our {ASM} method. {For a fair comparison, these methods all share the same training and testing settings.} Moreover, other recently proposed methods (i.e., Faster RCNN~\cite{ren2015faster} and SSD513~\cite{ssd16ECCV}) with the same ResNet architecture were also considered for comparison.

\subsection{Comparison Results}

Fig.~\ref{fig:cmp} illustrates the detection performance achieved using the FRCN pipeline with AlexNet~\cite{alexnet12NIPS} on the VOC 2007 test set. For a fair comparison, we initialize all the methods by providing only 5\% annotations, and allow FRCN+Ours, FRCN+CEAL and FRCN+ASPL to mine unlabeled samples only from the VOC 2007 train/val set. To fairly compare with FRCN+K-EM, we perform training and testing under a single image scale for the other methods. 

As shown, given a gradually increasing number of user annotations, our {ASM} method consistently performs better than the CEAL and K-EM approaches by clear margins. Specifically, our FRCN+{ASM} method can achieve a performance equivalent to that of a fully supervised method (i.e., FRCN with 100\% user annotations) with only approximately 30\% of the user annotations, whereas FRCN+CEAL and FRCN+K-EM require nearly 70\% user annotations. These results demonstrate the superior performance of our ASM.

\begin{figure*}[ptbp]
\center
\includegraphics[width=0.9 \textwidth]{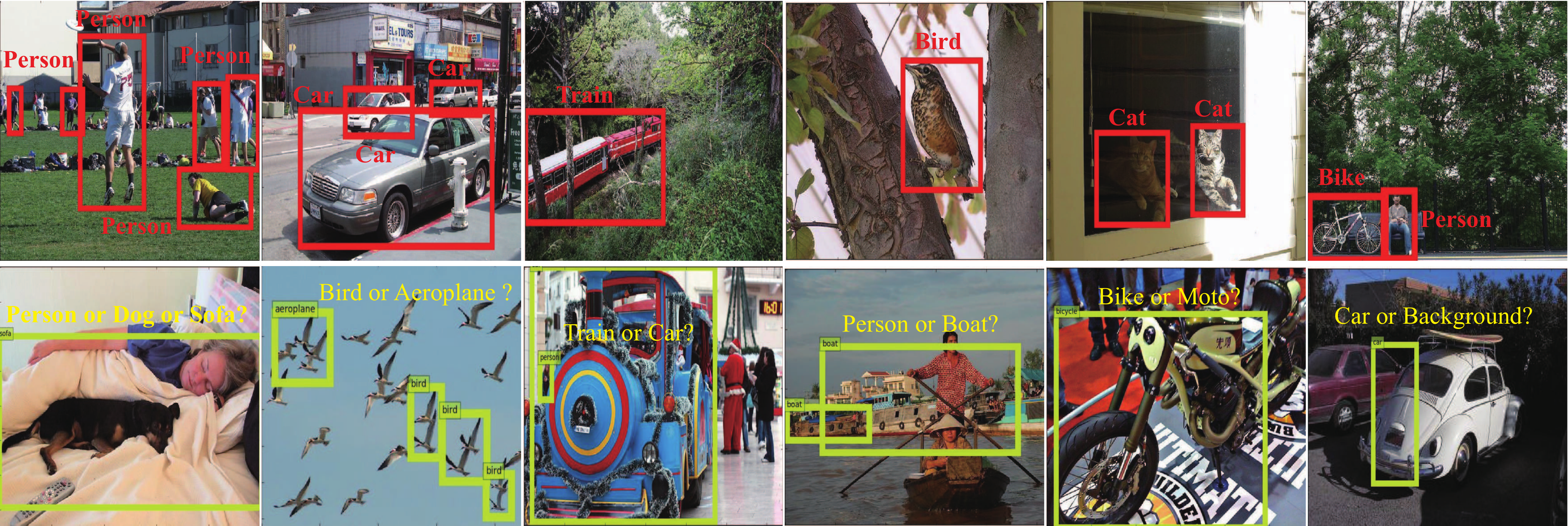}
\vspace{-5pt}
\caption{Selected examples from the COCO dataset. The first row shows high-confidence region proposals with pseudo-labels in red; the bottom row shows low-confidence region proposals in yellow, which required annotations from active users.}\label{fig:visual}
\vspace{-10pt}
\end{figure*}

To demonstrate the feasibility and great potential of our {ASM} approach for use on large-scale detection benchmarks, we conducted experiments on fine-tuning the RFCN model with ResNet-101~\cite{He_2016_CVPR} on the VOC 2007/2012 benchmark using our {ASM} approach and compared the results with those of the baseline RFCN with randomly selected annotations. The results obtained on the VOC 2007 and 2012 test sets are summarized in Tab.~\ref{tab:07mAP12} (a)(b), respectively. By controlling the number of training iterations, the performance of RFCN+{ASM} could be tested with different numbers of annotations (e.g., 20\%, 60\% and 100\% on the VOC 2007 test set). Note that the annotation percentages reported in these tables represent the additional annotations fetched from the VOC 2012 train/val sets for fine-tuning, in addition to the initial annotations (i.e., those from the VOC 2007 train/val sets), whereas in later tables, `pseudo' denotes the percentage of pseudo-labeled object proposals from VOC 2012 and COCO train/val images.

As the number of annotations increases, both RFCN+RAND and RFCN+{ASM} gradually achieve higher detection accuracy. However, as shown in Tab.~\ref{tab:07mAP12} (a)(b), our {ASM} approach consistently outperforms the baseline RFCN+RAND method under all annotation conditions {by clear margins on both the VOC 2007 and 2012 benchmarks}. These findings confirm the effectiveness of our {ASM} framework. Some examples of the selected high-confidence and low-confidence region proposals are depicted in Fig.~\ref{fig:visual}.

\subsection{Ablative Analysis}
To perform {a} component analysis of the proposed {ASM} {framework, we considered} variants using only the pseudo-labeling of high-confidence samples via the SL process and only the annotation of low-confidence samples via the AL process, denoted by ``FRCN+SL''/``RFCN+SL'' and ``FRCN+AL''/``RFCN+AL'', respectively. The only difference {between these methods} lies in how they treat the unlabeled object proposals. Taking the RFCN framework as an example, RFCN+SL learns in a purely self-learning fashion in accordance with Eqn.~(\ref{eq:ss}) until training ends, whereas RFCN+AL selectively collects low-confidence proposals, requests annotations, and stops when the annotation threshold is reached. Note that in this setting, the AL process cannot be used to guide the SL process because they are not jointly optimized.

\begin{table}[t]
\footnotesize
\center
\setlength{\tabcolsep}{2.5pt}\vspace{-5pt}
\caption{{Test set mAP results for VOC 2007 {obtained} using the RFCN~\cite{rfcn16NIPS} pipeline. Annotation key: `annotated' denotes the percentage of manual annotations used from the VOC 2007 train/val sets for initialization.}}\label{tab:07mAP07_fewshot}
\vspace{-5pt}
{\begin{tabular}{c | c | c }
\hline
\hline
Method &  annotated  & mAP \\
\hline
RFCN+MSPLD & 10\%  & 61.6 \\
RFCN+MSPLD & 30\%  & 68.2 \\
RFCN+MSPLD & 50\%   & 71.3 \\
RFCN+SL & 10\%   & 64.6	\\
RFCN+SL & 30\%   & 70.0	\\
RFCN+SL & 50\%   & 73.3 \\
\hline
RFCN & 100\% & \bf 73.9 \\
\hline
\hline
\end{tabular}}
\vspace{-10pt}
\end{table}



Tab.~\ref{tab:ss} (c) lists the mAP scores of the baseline RFCN method and RFCN+SL. As shown, given the same number of annotations during initialization, RFCN+SL performs significantly better than RFCN on both the VOC 2007 and VOC 2012 test sets. Specifically, RFCN+SL achieves a 3.6\% performance improvement (77.5\% vs. 73.9\%) by pseudo-labeling approximately 340\% of the high-confidence region proposals for training on the VOC 2007 benchmark, whereas a consistent performance gain of approximately 6\% (75.1\% vs. 69.1\%) is obtained on the VOC 2012 test set by pseudo-labeling high-confidence region proposals at a rate of approximately 130\%.
These results validate the significant contribution of the proposed SL process. Meanwhile, with the inclusion of the annotation of low-confidence samples via AL, RFCN+{ASM} performs slightly better than RFCN+SL on both the VOC 2007 and VOC 2012 test sets.

Moreover, we also compared our RFCN+SL method with a new state-of-the-art method, namely, RFCN+MSPLD~\cite{MSPLD}, {under} a few-shot object learning setting. {For a fair comparison,} the same percentage of manual annotations was used for model initialization in both our RFCN+SL method and the competing RFCN+MSPLD method. Then, the pseudo-labeling mechanism of each method was applied to fine-tune their models. The results for RFCN+MSPLD were obtained from the authors of \cite{MSPLD}. The results are compared in Tab.~\ref{tab:07mAP07_fewshot}. As shown in this table, RFCN+SL consistently outperforms RFCN+MSPLD by clear margins at all annotation percentages for model initialization. These findings further demonstrate the superiority of the proposed SL process.

To clarify the contribution of the proposed AL process, we conducted further experiments to compare the detection performances of RFCN+RAND and RFCN+AL under several different annotation-appending settings on the VOC 2007/2012 benchmarks. As shown by the results in Tab.~\ref{tab:al} (a)(b), RFCN+AL consistently outperforms the baseline RFCN+RAND, albeit by a small margin. Although the improvements are minor, the AL process is still beneficial in enhancing object detection. This slight improvement occurs because the informative samples with the greatest potential for improving performance lie in the long tail of the sample distribution, as reported in \cite{Wang_2017_CVPR}. Therefore, it is necessary either to obtain abundant training samples by asking active users to provide labels or to find other assistance. Fortunately, {the pseudo-labeling of high-confidence samples via our SL process is an effective way to address this issue.}

To prove that our {ASM framework} is critical and nontrivial, we also compared it with four simple and straightforward baselines, i.e., ``FRCN+AL+SL'', ``RFCN+AL+SL'', ``FRCN+SL+AL'' and ``RFCN+SL+AL'', by implementing the SL and AL processes in a simple sequential manner; e.g., RFCN+AL+SL first performs low-confidence sample annotation via AL and then runs using the SL process until training ends, whereas RFCN+SL+AL first learns using the same SL process as RFCN+SL and then continues in the same AL fashion as RFCN+AL. Therefore, these four methods combine the AL and SL processes {in a straightforward fashion rather than fusing them adaptively}. By contrast, under the control of the selector function, our RFCN+{ASM} method {selectively} switches between recognizing high-confidence proposals via the SL process and discovering low-confidence proposals under the proposed dual curricula for the next user annotation phase. The results, being listed in Tab.~\ref{tab:trivial} (c) and Tab.~\ref{tab:07mAP}, show that the proposed RFCN+{ASM} and FRCN+{ASM} methods outperform both of their corresponding baselines by clear margins. These findings validate the effectiveness of the AL and SL fusion design in our {ASM} framework.   

\subsection{Ablation Study Without Network Fine-tuning}
To {permit clear observation of} the effect of our {ASM} framework on sample mining in a fixed feature space, approaches without network fine-tuning were compared. Specifically, we pre-trained FRCN~\cite{frcn} with VGGNet~\cite{vgg} on the VOC 2012 train/val sets to obtain a good feature representation. Note that to avoid overusing the annotations from the VOC 2012 benchmark, we discarded the parameters of the softmax classifier layer from VGGNet and then used the results of employing 30\% and 57\% of the VOC 2007 train/val annotations to train the softmax classifier as a reference.  We initialized all methods with the same 30\% of the annotations from the VOC 2007 train/val sets and then allowed them to incrementally exploit the ``unlabeled'' samples (i.e., for which the annotations were not given). For all methods, the training was terminated when no low-confidence samples could be found (i.e., all unlabeled samples {had} been clearly classified).

As shown in Tab.~\ref{tab:07mAP}, our FRCN+{ASM} method achieves the highest mAP.  Specifically, FRCN+{ASM} obtains a mAP result that is 6.5\% and 4.3\% higher than those of FRCN when using 30\% and 57\% of the annotations, respectively. The FRCN+AL result demonstrates that AL can improve the detection mAP by using an additional 27\% of the annotations. The performance gain of FRCN+SL is marginal because some of the category classifiers are heavily degraded (such as the bottle, person and plant category classifiers). This degradation occurs because certain classifiers are easily misled by outliers without active user intervention. Similar degradation also occurs in FRCN+SL+AL. Our FRCN+{ASM} method outperforms FRCN+SL by a clear margin using only an additional 5\% of the annotations, and it performs 2\% better than FRCN+AL using 12\% fewer annotations. In summary, our FRCN+{ASM} method achieves a higher mAP score while requiring fewer annotations, thereby demonstrating its superior performance.

\begin{table}[t]
\footnotesize
\center
\setlength{\tabcolsep}{2pt}
\caption{Test set mAP results for VOC 2007 obtained using the FRCN~\cite{ren2015faster} pipeline with VGGNet. Annotation key: `append' denotes the number of appended annotations requested via AL as a percentage of all annotations from the VOC 2007 train/val sets. For all methods, 30\% of the annotations from the VOC 2007 train/val sets were used for initialization.}\label{tab:07mAP}
\vspace{-8pt}
\begin{tabular}{c | c | c }
\hline
\hline
Method & append & mAP \\
\hline
FRCN  & 0\% &  62.0 \\
FRCN+RAND  & 27\% & 64.2  \\
FRCN+SL  & 0\% & 62.5 \\
FRCN+AL  & 27\% & 66.3 \\
FRCN+SL+AL  & 12\%  & 65.4\\
FRCN+AL+SL  & 27\% & 67.7  \\
FRCN+{ASM} & 5\% & \bf 68.5 \\
\hline
\hline
\end{tabular}
\vspace{-5pt}
\end{table}

\begin{figure}[tbp]
\center
\includegraphics[width = 0.6 \columnwidth]{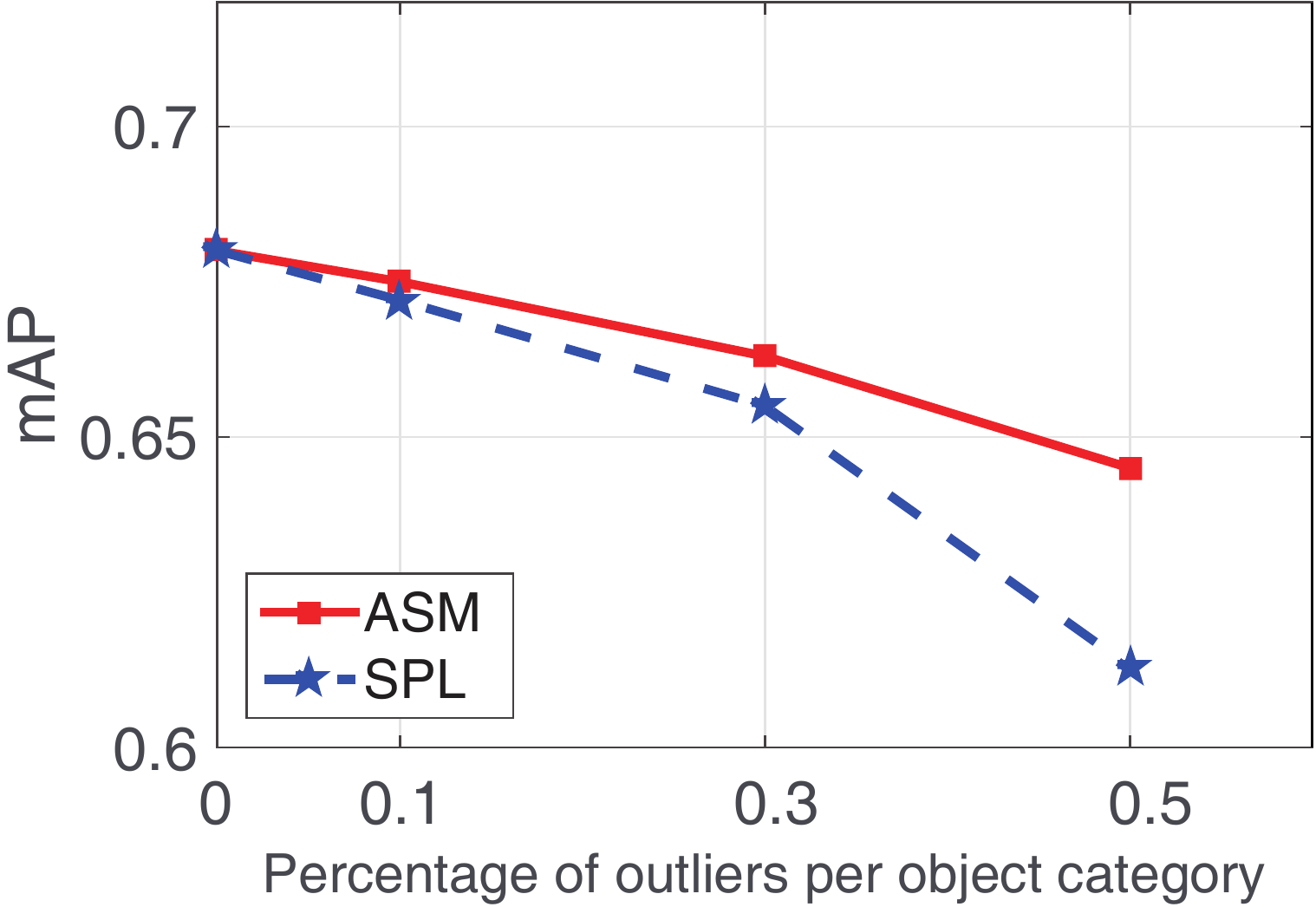}
\vspace{-5pt}
\caption{The accuracy-outlier curves obtained on the VOC 2007 test set by adding outliers to the training samples.}
\vspace{-8pt}
\label{fig:aoc}
\end{figure}

\begin{table}[t]
\footnotesize
\center
\setlength{\tabcolsep}{2pt}
\caption{{Test set mAP results for VOC 2007 obtained using the RFCN~\cite{rfcn16NIPS} pipeline. Annotation key: `annotated' denotes the percentage of manual annotations used from the VOC 2007 train/val and VOC 2012 train/val sets; `pseudo' denotes the percentage of pseudo-labeled region proposals relative to the number of VOC 2007 train/val annotations.}}\label{tab:time}
\vspace{-8pt}
{\begin{tabular}{c |c  c | c | c c}
\hline
\hline
Method & annotated & pseudo & mAP & Training Time & Testing Time \\
& & & & (seconds/image) & (seconds/image) \\
\hline
RFCN & 100\% & 0\% & 79.8 & 0.42 & 0.12  \\
RFCN+{ASM} & 30\% & 400\% & 79.3 & 0.44 & 0.12 \\
RFCN+{ASM} & 50\% & 500\% & 79.9 & 0.46 & 0.12 \\
RFCN+{ASM} & 50\% & 600\% & 80.9 & 0.49 & 0.12 \\
RFCN+{ASM} & 100\% & 1000\% &\bf 81.8 & 0.63 & 0.12 \\
\hline
\hline
\end{tabular}}
\vspace{-10pt}
\end{table}

\begin{figure*}[tbp]
\center
\includegraphics[width = 0.7 \textwidth]{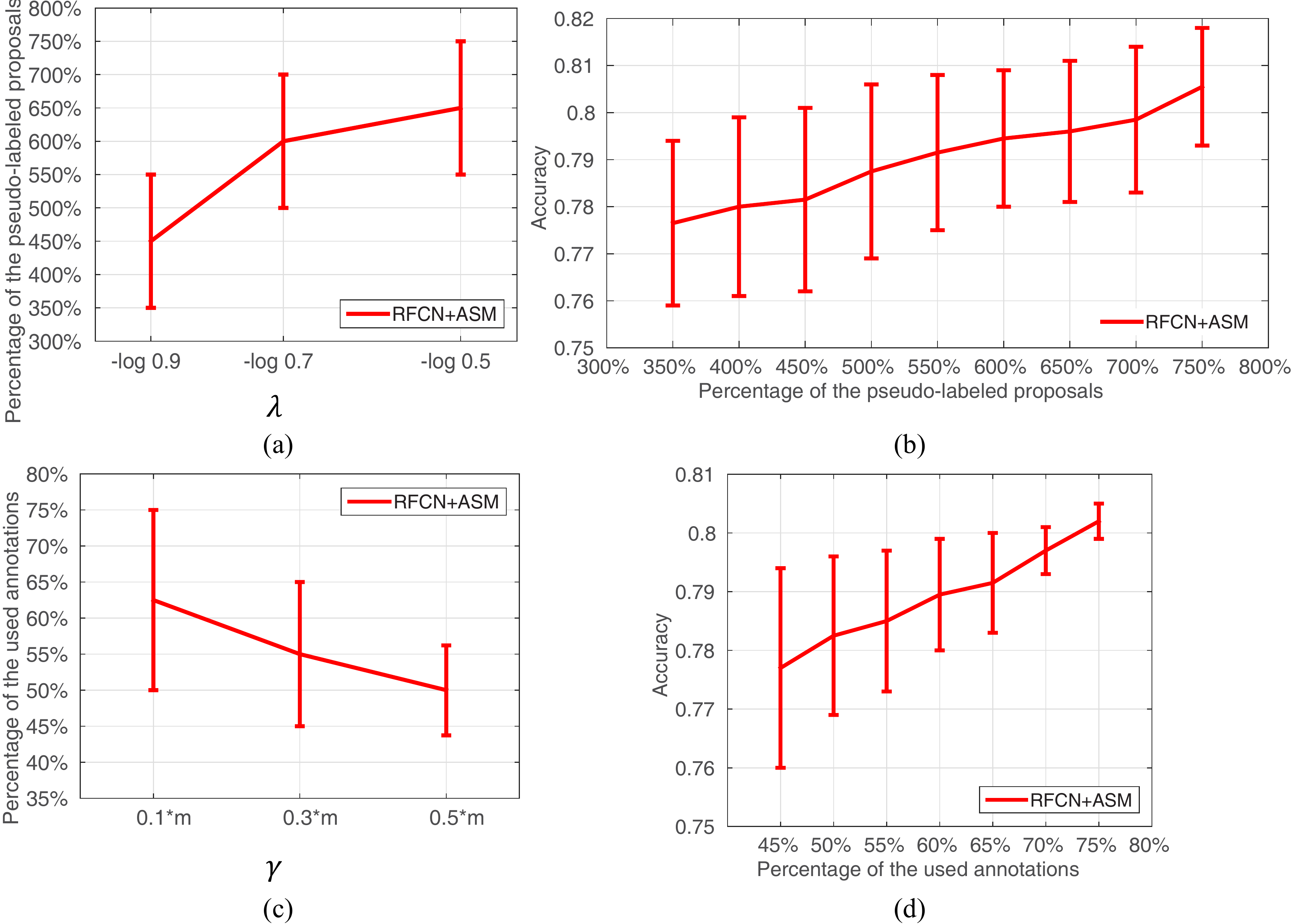}
\vspace{-10pt}
\caption{{Sensitivity analysis of the hyperparameters (a) $\lambda_0$ and (c) $\gamma$. As shown, a larger $\lambda_0$ represents a higher threshold for defining high-confidence samples, whereas a larger $\gamma$ represents a higher threshold for selecting low-confidence samples. Meanwhile, (b) illustrates the detection accuracy corresponding to (a) under different percentages of pseudo-labels, and (d) illustrates the detection accuracy corresponding to (c) under different percentages of used annotations.} }\label{fig:hyper}
\vspace{-10pt}
\end{figure*}

\subsection{Robustness Analysis}
To further demonstrate the potential of our {ASM} framework to suppress the influence of outliers/noisy samples (including samples belonging to undefined object categories), we measured the robustness of our method against outliers. We employed the {\em accuracy-outlier curve} (AOC) to evaluate the robustness of our method and the SPL baseline method reported in~\cite{spcl}. The AOC plots the accuracy w.r.t the percentage of outliers per object category. In this analysis, we considered three sources of outliers: i) a number of samples with manual annotations {were randomly selected and assigned uncorrected} category labels with respect to the ground-truth annotations to generate noisy samples/outliers, ii) incorrectly annotated samples among the manual annotations were treated as outliers, and iii) automatically pseudo-labeled samples belonging to the undefined object category were also considered to be outliers. To ensure a reasonable evaluation, we initialized both our {ASM} and SPL methods using the same network parameters under the FRCN pipeline. Then, we allowed the parameters of our {ASM} and SPL models to be trained in their own ways with different percentages of outliers per object category. As illustrated in Fig.~\ref{fig:aoc}, our {ASM} method is much more robust than the SPL method under different outlier percentages per object category. By incorporating AL-based guidance for the SL process, our {ASM} framework achieves stable results when faced with outliers and noisy samples.

\subsection{Time Efficiency Analysis}
We also compared the efficiency of our ASM framework and the original RFCN model in both the training and testing phases. The results of this comparison are presented in Tab.~\ref{tab:time}, which shows that due to sharing the same object detection pipeline, our RFCN+{ASM} and the original RFCN have identical average time costs for testing a given image under a single image scale without flipping. It is obvious that our RFCN+{ASM} (30\% used annotations + 400\% pseudo-labels) and the original RFCN also have approximately the same time cost for training. These findings confirm that the additional time complexity introduced by our model can be ignored, thanks to our proposed closed-form solution for updating the latent weight variables $\mathbf{U}$ and $\mathbf{V}$. 

However, as the threshold parameters ${\bm \lambda}$ increase, more high-confidence region proposals from each unlabeled or partially labeled image will be pseudo-labeled for network fine-tuning. This will result in a higher computational cost for network fine-tuning. Nevertheless, our model requires only a 50\% increase in time cost (0.63 vs. 0.42 second/image) to reach the 1000\% pseudo-annotation objective, which offers a significant performance gain (approximately 2\% mAP) against the original RFCN. Therefore, the additional time complexity introduced by our model is moderate and acceptable.

\begin{table}[!t]
\footnotesize
\center
\setlength{\tabcolsep}{1.5pt}
\caption{{Test set mAP results for VOC 2007 obtained using the RFCN~\cite{rfcn16NIPS} pipeline. Annotation key: `m' denotes the number of defined object categories; `annotated' denotes the percentage of manual annotations used from the VOC 2007 train/val sets.}}\label{tab:07mAP12_unseen}
\vspace{-5pt}
{\begin{tabular}{c | c c | c | c }
\hline
\hline
Method & $m$ & unseen & annotated & mAP \\
\hline
RFCN+CEAL & 15 & 5 & 30\% & 70.1 \\
RFCN+ASPL & 15 & 5 & 30\% & 75.5 \\
RFCN+{ASM} & 15 & 5 & 30\% & 78.7  \\
RFCN & 15 & 0 & 100\% &\bf  79.7 \\
\hline
RFCN+CEAL & 20 & 5 & 30\% & 73.5 \\
RFCN+ASPL & 20 & 5 & 30\% & 76.0 \\
RFCN+{ASM} & 20 & 5 & 30\% & 78.1  \\
RFCN & 20 & 0 & 100\% &\bf  79.3 \\
\hline
\hline
\end{tabular}}
\vspace{-10pt}
\end{table}

\subsection{Unseen Object Category Support}
To confirm the effectiveness of our model in supporting unseen object categories, we conducted two further evaluations on the VOC 2007 benchmark, which contains 20 object categories. In the first evaluation, we treated only 15 of these object categories as valid categories for detection, and the remaining 5 categories were treated as unseen ones. In the second evaluation, we additionally added several samples in 5 new categories from the Microsoft COCO dataset into the training set as unseen noisy samples. 

To demonstrate the superior performance of our ASM framework, we compared RFCN+{ASM} with the closely related methods ``RFCN+CEAL'' and ``RFCN+ASPL''. To ensure a fair comparison, all methods were initialized with the same annotations (i.e., 10\%) and allowed to fetch the same number of additional annotations (i.e., 20\%) during the AL process. Thus, the remaining 70\% samples were treated as unlabeled. The only difference among these methods is how they select and assign pseudo-labels to these unlabeled samples. The results of these evaluations are compared in Tab.~\ref{tab:07mAP12_unseen}. 

In contrast to the upper-bound baseline ``RFCN (ALL)'', in which all VOC 2007 and VOC 2012 train/val annotations (without unseen categories) were used for training, all methods tested here employed only 30\% of the annotations and assigned pseudo-labels to the remaining unlabeled samples for fine-tuning of the network. As shown in Tab.~\ref{tab:07mAP12_unseen}, the performance of our proposed method is inferior to that of RFCN (ALL) by approximately 1\% and consistently superior to that of RFCN+CEAL and RFCN+ASPL by clear margins in both evaluation settings. These findings confirm the superior ability of our method to overcome the misleading influence of samples that belong to unseen object categories.

\subsection{Hyperparameter Sensitivity}
We further analyzed the sensitivity of hyperparameters $\lambda_0$ and $\gamma$. Specifically, we trained our model with $\lambda_0$ values ranging from -$\log 0.9$ to -$\log 0.5$ with $\gamma$ fixed and with $\gamma$ values ranging from 0.1$\times m$ to 0.5$\times m$ with $\lambda_0$ fixed to perform the sensitivity analyses. The results of the sensitivity analysis for $\lambda_0$ are shown in Fig.~\ref{fig:hyper} (a) and Fig.~\ref{fig:hyper} (b), whereas those for $\gamma$  are shown in Fig.~\ref{fig:hyper} (c) and Fig.~\ref{fig:hyper} (d). $\lambda_0$ serves as the threshold for defining high-confidence samples during model initialization and thus controls the percentage of region proposals that are assigned pseudo-labels. As shown in Fig.~\ref{fig:hyper} (a), far more pseudo-labeled proposals are obtained as $\lambda_0$ increases. Fig.~\ref{fig:hyper} (b) demonstrates that the detection accuracy continuously increases as the percentage of pseudo-labeled proposals grows. Note that although more pseudo-labeled proposals can substantially improve the model performance, the standard deviation of the detection accuracy is still high due to the misleading influence of incorrectly pseudo-labeled proposals. Meanwhile, $\gamma$ serves as the threshold for the selection of low-confidence samples during model initialization and thus is relates to the percentage of annotations used. As shown in Fig.~\ref{fig:hyper} (c), a smaller $\gamma$ value results in requests for more manual annotations. Fig.~\ref{fig:hyper} (d) illustrates that more annotations can lead to a higher detection accuracy with a lower standard deviation. To ensure the {effectiveness of automatically pseudo-labeling} and the cost-effective manual annotation of the representative minority samples, we empirically set \{$\lambda_0$, $\gamma$\}=\{$-\log 0.9$, $0.5m$\}.

\section{Conclusions}
\label{sec:conc}
In this paper, we have introduced a principled active sample mining framework and demonstrated its effectiveness in mining the majority of unlabeled or partially labeled data to boost object detection. In our {ASM} framework, a {self-learning} process, integrated into the AL pipeline with a concise formulation, is employed for retraining the object detectors using {accurately pseudo-labeled object proposals}. Meanwhile, the remaining samples {with low prediction confidence (i.e., high uncertainty) by the current detectors can be annotated} through {the AL process}, which {contributes to} generating reliable and diverse samples and gradually revising the {self-learning} process. By means of the proposed alternating optimization mechanism, our framework {selectively} and seamlessly switches between our {self-learning} process and the AL process for each unlabeled or partially labeled sample. Moreover, two curricula are introduced to guide the pseudo-labeling and annotation processes from dual perspectives. Thus, our {ASM} framework can be used to build effective CNN detectors that require fewer labeled training instances while achieving promising results. In the future, we plan to extend our framework to achieve improvements in other specific types of visual detection using unlabeled videos under the large-scale application scenarios.

\appendices
\section{Proof of Proposition 1}
\label{sec:p1}
$ \mathbb{E}(\mathbf{U},\mathbf{V}; \mathbf{L}(X, \mathbf{Y}; \mathbf{W}),\gamma,{\bm \lambda}) $ can be decoupled as $\sum^n_{i=1} E(u_i,\mathbf{v}_i;\mathcal{L}_i,\gamma, {\bm \lambda})$.
		Since $\forall \ i_i \neq i_2$, $E(u_{i_1},v_{i_1};\mathcal{L}_{i_1},\gamma,{\bm \lambda})$ and $E(u_{i_2},v_{i_2};\mathcal{L}_{i_2},\gamma,{\bm \lambda})$ are independent, we can have   
		\begin{small}
			\begin{displaymath}
			\begin{aligned}
			\min_{\mathbf{V}\in\mathbf{\Psi}^{\bm \lambda}}\mathbb{E}(\mathbf{U},\mathbf{V};\mathbf{L}(X, \mathbf{Y}; \mathbf{W}),\gamma,{\bm \lambda}) &= \sum^n_{i=1} \min_{\mathbf{v}_i\in V^{\bm \lambda}_i} E(u_i,\mathbf{v}_i;\mathcal{L}_i,\gamma,{\bm \lambda}) 
			\\ &= \sum^n_{i=1}  E(u_i,\mathbf{v}^\ast_i;\mathcal{L}_i,\gamma,{\bm \lambda}).
			\end{aligned}
			\end{displaymath}
		\end{small}Furthermore, $\mathbf{V}$ and $\mathbf{U}$ are also independent of each other. Thus, we can similarly obtain
		\begin{small}
			\begin{displaymath}
			\max_{\mathbf{U}\in\mathbf{\Psi}^{\bm \gamma}}\mathbb{E}(\mathbf{U},\mathbf{V}^\ast;\mathbf{L}(X, \mathbf{Y}; \mathbf{W}),\gamma,{\bm \lambda}) = \sum^n_{i=1} \max_{u_i\in U^{\gamma}_i} E(u_i,\mathbf{v}^\ast_i;\mathcal{L}_i,\gamma,{\bm \lambda}).
			\end{displaymath}
		\end{small}
		Hence, we have $\underset{\mathbf{U}\in\mathbf{\Psi}^{\bm \gamma}}{\max} \ \underset{\mathbf{V}\in\mathbf{\Psi}^{\bm \lambda}}{\min} \ \mathbb{E}(\mathbf{U},\mathbf{V};\mathbf{L}(X, \mathbf{Y}; \mathbf{W}),\gamma,{\bm \lambda})$
		\begin{small}
			\begin{displaymath}
			\begin{aligned}
			=
			\underset{\mathbf{U}\in\mathbf{\Psi}^{\bm \gamma}}{\max}\ \sum^n_{i=1} & E(u_i,\mathbf{v}^\ast_i;\mathcal{L}_i,\gamma,{\bm \lambda})
			\\=\sum^n_{i=1} \max_{u_i\in U^\gamma_i}  & \min_{\mathbf{v}_i\in V^{\bm \lambda}_i} E(u_i,\mathbf{v}_i;\mathcal{L}_i,\gamma,{\bm \lambda}).
			\end{aligned}
			\end{displaymath}
		\end{small}
Therefore, optimizing Eqn.~(\ref{eq:obj}) is equivalent to {performing} $\underset{u_i\in U^\gamma_i}{\max} \underset{\mathbf{v}_i\in V^{\bm \lambda}_i}{\min} \  E(u_i,\mathbf{v}_i;\mathcal{L}_i,\gamma,\lambda)$ on each sample $x_i$ in $X$. 
	
{
\section{Proof of Proposition 2}
\label{sec:p2}
	
	Based on Proposition 1, we consider $x_i$ in $X$. Thus, the min-max problem expressed in Eqn.~(\ref{eq:uv}) is transformed as:	
	\begin{small}\begin{equation}\begin{aligned}
		\underset{u_i}{\max} \ \underset{\mathbf{v}_i}{\min} \  E(u_i&,\mathbf{v}_i;\mathcal{L}_i,\gamma,{\bm \lambda}) \\= \sum_{j=1}^{m}\max(u_i,v_i^{(j)})l_i^{(j)} -\gamma u_i &+ \frac{1}{2}\overset{m}{\underset{j=1}{\sum}}\lambda^{(j)}((v_i^{(j)})^2-2v_i^{(j)}) \\  \mathbf{s.t.} \ \ \gamma > 0, u_i \in& \{0,1\}\cap U^{\gamma}_i, \\ \lambda^{(j)} > 0; \ \mathbf{v}_i=\{v_i^{(j)}\}^m_{j=1}& \in [0,\epsilon]^m\cap V^{\bm \lambda}_i \subset [0,1)^m.
		\label{eq6}\end{aligned}
		\end{equation}\end{small}$\forall x_i\in A_{t-1}\cup B_{t-1}$, the $\{u_i,\mathbf{v}_i\}$ are constant. If $x_i\in X/(A_{t-1}\cup B_{t-1})$, then $\{0,1\}\cap U^{\gamma}_i=\{0,1\}, [0,\epsilon]^m\cap V^{\bm \lambda}_i = [0,\epsilon]^m$. Hence, we can solve for the $(u_i,\mathbf{v}_i)$ via Proposition 2. 
		
	When $u_i=1>\max\{v_i^{(j)}\}^m_{j=1}$, Eqn.~(\ref{eq5}) yields
	\begin{small}
		\begin{equation}
		\begin{aligned}
		E_i = u_i\sum_{j=1}^{m}l_i^{(j)} -&\gamma u_i + \frac{1}{2} \overset{m}{\underset{j=1}{\sum}}\lambda^{(j)}((v_i^{(j)})^2-2v_i^{(j)}),
		\end{aligned}~\label{r1}
		\end{equation}
	\end{small}and 
	\begin{small}
		\begin{equation}
		\begin{aligned}
		\frac{\partial E_i}{\partial v^{(j)}_i} = \lambda(v^{(j)}_i-1).
		\end{aligned}~\label{d1}
		\end{equation}
	\end{small}
	When $u_i=0\leq\min\{v_i^{(j)}\}^m_{j=1}$, Eqn.~(\ref{eq5}) yields
	\begin{small}
		\begin{equation}
		\begin{aligned}
		E_i = \sum_{j=1}^{m}v_i^{(j)}l_i^{(j)}& + \frac{1}{2} \overset{m}{\underset{j=1}{\sum}}\lambda^{(j)}((v_i^{(j)})^2-2v_i^{(j)}),
		\end{aligned}~\label{r2}
		\end{equation}
	\end{small}and
	\begin{small}
		\begin{equation}
		\begin{aligned}
		\frac{\partial E_i}{\partial v^{(j)}_i} = l^{(j)}_i+\lambda(v^{(j)}_i-1).
		\end{aligned}~\label{d2}
		\end{equation}
	\end{small}	
	\begin{lem}
		Given Eqn.~(\ref{r2}), with respect to $\mathbf{v}$, the solution is		
		\begin{small}
			\begin{equation}
			\begin{gathered}
			v_i^{(j)} 
			=\left\{
			\begin{aligned}	
			&0 &l^{(j)}_i>\lambda^{(j)}; \\
			&1-\frac{l^{(j)}_i}{\lambda^{(j)}}\ \  &\lambda^{(j)}(1-\epsilon)\leq l^{(j)}_i\leq\lambda^{(j)}; \\
			&\epsilon &l^{(j)}_i<\lambda^{(j)}(1-\epsilon).
			\end{aligned}
			\right.
			\end{gathered}\label{v1}
			\end{equation}
		\end{small}	\label{lem1}	
	\end{lem}
	\begin{proof}
		When $l^{(j)}_i>\lambda^{(j)}$, 
		\begin{small}
			\begin{displaymath}
			\begin{aligned}
			\frac{\partial E_i}{\partial v^{(j)}_i} &= l^{(j)}_i+\lambda^{(j)}(v^{(j)}_i-1) > \lambda^{(j)} + \lambda^{(j)}(v^{(j)}_i-1) \geq 0,
			\end{aligned}
			\end{displaymath}
		\end{small}which leads to $\arg\underset{v\in[0,\epsilon]}{\min} E_i = 0$. When $l^{(j)}_i\leq\lambda^{(j)}(1-\epsilon)$,
		\begin{small}
			\begin{displaymath}
			\begin{aligned}
			\frac{\partial E_i}{\partial v^{(j)}_i} =l^{(j)}_i+\lambda(v^{(j)}_i-1) &< \lambda^{(j)}(1-\epsilon) + \lambda^{(j)}(v^{(j)}_i-1) \\ &\leq \lambda^{(j)}(v^{(j)}_i-\epsilon) \leq 0,
			\end{aligned}
			\end{displaymath}
		\end{small}which leads to $\arg\underset{v\in[0,\epsilon]}{\min} E_i = \epsilon$. When $\lambda^{(j)}(1-\epsilon)\leq l^{(j)}_i\leq\lambda^{(j)}$, from $\frac{\partial E_i}{\partial v^{(j)}_i} = 0$, we have $\arg\underset{v\in[0,\epsilon]}{\min} E_i = 1-\frac{l^{(j)}_i}{\lambda^{(j)}}\in[0,\epsilon]$. This concludes the proof.

	\end{proof}
\begin{lem}
		Given any iteration $t$, for the sample $x_i$, suppose that $u_i^t\in\{0,1\}$ and $\mathbf{v}_i^t\in [0,\epsilon]^m \subset [0,1)^m$. Then, $\sum_{j=1}^{m}l_i^{(j)} > \frac{\gamma}{1-\epsilon}$ leads to $u_i^{t+1}=u_i^{t+2}=1$, and $\sum_{j=1}^{m}l_i^{(j)} < \gamma$ leads to $u_i^{t+1}=u_i^{t+2}= 0$.\label{lem2}
	\end{lem}
	
\begin{proof}
		Both cases are discussed below.
		
		1). $\sum_{j=1}^{m}l_i^{(j)} > \frac{\gamma}{1-\epsilon}$: Consider $u_i^t=1$. Since $\frac{\partial E_i}{\partial v^{(j)}_i}<0$, we have $\forall j$, $v_i^{(j)(t+1)}=\epsilon<1$. From $\mathbf{v}_i^{t+1}$ to $u_i^{t+1}$, we consider Eqn.~(\ref{eq5}) with respect to $u_i$ as follows:	
		\begin{small}
			\begin{equation}
			\begin{gathered}
			E(u_i^{t+1},\mathbf{v}_i^{t+1};\mathcal{L}_{i},\gamma,{\bm \lambda})  \\
			=\left\{
			\begin{aligned}	
			&\epsilon\sum_{j=1}^{m}l_i^{(j)} + f_{SL}(\mathbf{v}^{(t+1)}_i, \bm \lambda) \ \  &u_i^{t+1}=0; \\
			&\sum_{j=1}^{m}l_i^{(j)} -\gamma + f_{SL}(\mathbf{v}^{(t+1)}_i, \bm \lambda)\ \  &u_i^{t+1}=1. 
			\end{aligned}
			\right.
			\end{gathered}\label{u}
			\end{equation}
			\begin{displaymath}
			\begin{aligned}
			&E_{u_i^{t+1}=1}-E_{u_i^{t+1}=0} = \sum_{j=1}^{m}l_i^{(j)} -\gamma - \epsilon\sum_{j=1}^{m}l_i^{(j)} \\&= (1-\epsilon)\sum_{j=1}^{m}l_i^{(j)}-\gamma > (1-\epsilon)\frac{\gamma}{1-\epsilon}-\gamma = 0, 
			\end{aligned}
			\end{displaymath}
		\end{small}which leads to $u^{t+1}_i=1$ and $u^{t+2}_i=1$. Consider $u_i^t=0$. From Lemma~\ref{lem1}, it holds that $\forall j$, $v_i^{(j)(t+1)} \geq u_i^t$, and Eqn.~(\ref{eq5}) with respect to $u_i$ is presented as follows:	
			\begin{small}
			\begin{equation}
			\begin{gathered}
			E(u_i^{t+1},\mathbf{v}_i^{t+1};\mathcal{L}_{i},\gamma,{\bm \lambda})  \\
			=\left\{
			\begin{aligned}	
			&\sum_{j=1}^{m}v^{(j)(t+1)}_il_i^{(j)} + f_{SL}(\mathbf{v}^{(t+1)}_i, \bm \lambda) \ \  &u_i^{t+1}=0; \\
			&\sum_{j=1}^{m}l_i^{(j)} -\gamma + f_{SL}(\mathbf{v}^{(t+1)}_i, \bm \lambda)\ \  &u_i^{t+1}=1. 
			\end{aligned}
			\right.
			\end{gathered}\label{v2}
			\end{equation}
		\end{small}		
		\begin{small}
			\begin{displaymath}
			\begin{aligned}
			&E_{u_i^{t+1}=1}-E_{u_i^{t+1}=0} \\
			 &= \sum_{j=1}^{m}l_i^{(j)} -\gamma - \sum_{j=1}^{m}v^{(j)(t+1)}_il_i^{(j)} 
			\geq \sum_{j=1}^{m}l_i^{(j)} -\gamma - \epsilon\sum_{j=1}^{m}l_i^{(j)}\\
			&= (1-\epsilon)\sum_{j=1}^{m}l_i^{(j)}-\gamma > (1-\epsilon)\frac{\gamma}{1-\epsilon}-\gamma = 0, 
			\end{aligned}
			\end{displaymath}
		\end{small}which leads to $u^{t+1}_i=1$ and $u^{t+2}_i=1$. 
		
			2). $\sum_{j=1}^{m}l_i^{(j)} < \gamma$: Consider $u_i^t=0$. Then, Eqn.~(\ref{eq6}) transforms into Eqn.~(\ref{r2}). With respect to $\mathbf{v}_i$, we follow the solution given in Eqn.~(\ref{v1}) and consider $u^{t+1}_i$ in Eqn.~(\ref{v2}).
				\begin{small}
					\begin{displaymath}
					\begin{aligned}
					E_{u_i^{t+1}=1}-E_{u_i^{t+1}=0} &= \sum_{j=1}^{m}l_i^{(j)} -\gamma - \sum_{j=1}^{m}v^{(j)(t+1)}_il_i^{(j)} \\
					&\leq \sum_{j=1}^{m}l_i^{(j)} -\gamma< 0,
					\end{aligned}
					\end{displaymath}
				\end{small}which leads to $u^{t+1}_i=0$ and $u^{t+2}_i=0$. Consider $u_i^t=1$; then, we similarly obtain $\mathbf{v}_i^{t+1}=\{\epsilon\}^m$. According to Eqn.~(\ref{u}), $E_{u_i^{t+1}=1}-E_{u_i^{t+1}=0} = \sum_{j=1}^{m}l_i^{(j)} -\gamma - \epsilon\sum_{j=1}^{m}l_i^{(j)} \leq\sum_{j=1}^{m}l_i^{(j)} -\gamma< 0$; thus, it holds that $u_i^{t+1}=0$ and $u_i^{t+2}=0$.
		
		The analyses of 1) and 2) together conclude the proof.
	    \end{proof}

Based on Lemmas~\ref{lem1} and \ref{lem2}, we present the proof of Proposition 2 as follows.
	    
\begin{proof}
Let $u_0$ and $\mathbf{v}_{0}$ denote the initializations of $u_i$ and $\mathbf{v}_i$, respectively. Following Lemma~\ref{lem2}, when $\forall~t\geq 2$, the following conclusion holds:
		\begin{small}
			\begin{equation}
			\begin{aligned}
			u_i^\ast=\lim_{t \to +\infty}u_i^t=\left\{
			\begin{aligned}
			1  \ \ &\sum_{j=1}^{m}l_i^{(j)} > \frac{\gamma}{1-\epsilon}; \\
			0  \ \ &\sum_{j=1}^{m}l_i^{(j)} < \gamma. \\  
			\end{aligned}
			\right.
			\end{aligned}\label{eq8}
			\end{equation}
		\end{small}Here, we consider $\mathbf{v}^\ast$. When $\sum_{j=1}^{m}l_i^{(j)} > \frac{\gamma}{1-\epsilon}$ , it holds that $u_i^{t}=u_i^{t+1}=1$. Then, $\forall j$, consider $v_i^{(j)t}$ as follows:
		\begin{small}
\begin{displaymath}
\begin{aligned}
\underset{\mathbf{v}_i^t}{\min} \ E(u_i^t,\mathbf{v}_i^t;\mathcal{L}_i&,\gamma,{\bm \lambda}) \\= \underset{\{v_i^{(j)t}\}^m_{j=1}}{\min}\sum_{j=1}^{m}\max(u_i^t,v_i^{(j)t})l_i^j &+ f_{AL}(u_i^t,\gamma) + f_{SL}(\mathbf{v}_i^t,{\bm \lambda})\\=
\sum_{j=1}^{m}\underset{v_i^{(j)t}}{\min}( \frac{1}{2}\lambda^{(j)}((v_i^{(j)t})^2-2v_i&^{(j)t})) + u_i^t\sum_{j=1}^{m}l_i^{(j)} + f_{AL}(u_i^t,\gamma).
\end{aligned}
\end{displaymath}
		\end{small}Then, $\forall j$, we solve $\underset{v_i^{(j)t}}{\min}\frac{1}{2}\lambda^{(j)}((v_i^{(j)})^2-2v_i^{(j)})$ and obtain $v_i^{(j)t}=\epsilon$. The same result is found for $v_i^{(j)(t+1)}$, leading to
		\begin{small}
			\begin{equation}
			\begin{aligned}
			\mathbf{v}_i^\ast=\lim_{t \to +\infty}\mathbf{v}_i^t=
			\epsilon.  \\  
			\end{aligned}\label{eq9}
			\end{equation}
		\end{small}When $\sum_{j=1}^{m}l_i^{(j)} < \gamma$, it holds that $u_i^{t}=u_i^{t+1}=0$. Thus, $\forall j\in[m]$, we consider $v_i^{(j)t}$ as below:
		\begin{small}
			\begin{displaymath}
			\begin{aligned}
			\underset{\mathbf{v}_i^t}{\min} \ E(u_i^t,\mathbf{v}_i^t;\mathcal{L}_i&,\gamma,{\bm \lambda}) \\= \underset{\{v_i^{(j)t}\}^m_{j=1}}{\min}\sum_{j=1}^{m}\max(u_i^t,v_i^{(j)t})l_i^{(j)} &+ f_{AL}(u_i^t,\gamma) + f_{SL}(\mathbf{v}_i^t,{\bm \lambda})\\=
			\underset{\{v_i^{(j)t}\}^m_{j=1}}{\min}\sum_{j=1}^{m}v_i^{(j)t}l_i^{(j)} + \frac{1}{2}\lambda^{(j)}(&(v_i^{(j)t})^2-2v_i^{(j)t}) + f_{AL}(u_i^t,\gamma) \\=
			\sum_{j=1}^{m}\underset{v_i^{(j)t}}{\min}(v_i^{(j)t}l_i^{(j)} + \frac{1}{2}\lambda^{(j)}((v_i&^{(j)t})^2-2v_i^{(j)t})) + f_{AL}(u_i^t,\gamma).
			\end{aligned}
			\end{displaymath}
		\end{small}Then, $\forall j$, we optimize $\underset{v^{(j)t}_i}{\min}(v^{(j)t}_il_i^{(j)} + \frac{1}{2}\lambda^{(j)}((v_i^{(j)t})^2-2v_i^{(j)t}))$. The solution is given by Lemma~\ref{lem1}, which also holds $\forall t'>t$. We conclude that the proposition is justified by Eqns.~(\ref{eq8})--(\ref{eq9}) and Lemma~\ref{lem1}.
\end{proof}
}

\ifCLASSOPTIONcaptionsoff
  \newpage
\fi



%
\bibliographystyle{IEEEtran}
\bibliography{asm}

\begin{IEEEbiography}[{\includegraphics[width=1in,height=1.25in,clip,keepaspectratio]{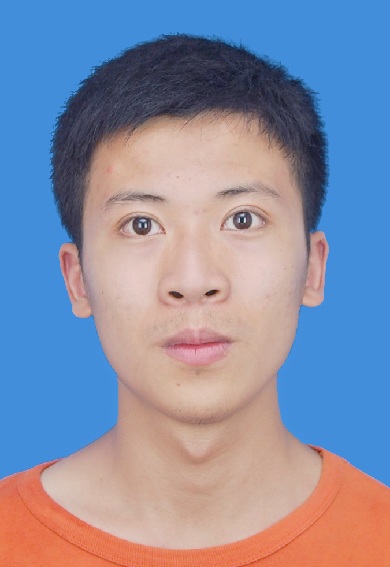}}]{Keze Wang} received his B.S. degree in software engineering from Sun Yat-Sen University, Guangzhou, China, in 2012. He is currently pursuing the dual Ph.D. degree at Sun Yat-Sen University and Hong Kong Polytechnic University, advised by Prof. Liang Lin and Lei Zhang. His current research interests include computer vision and machine learning. More information can be found in his personal website \url{http://kezewang.com}.
\end{IEEEbiography}

\begin{IEEEbiography}[{\includegraphics[width=1in,height=1.25in,clip,keepaspectratio]{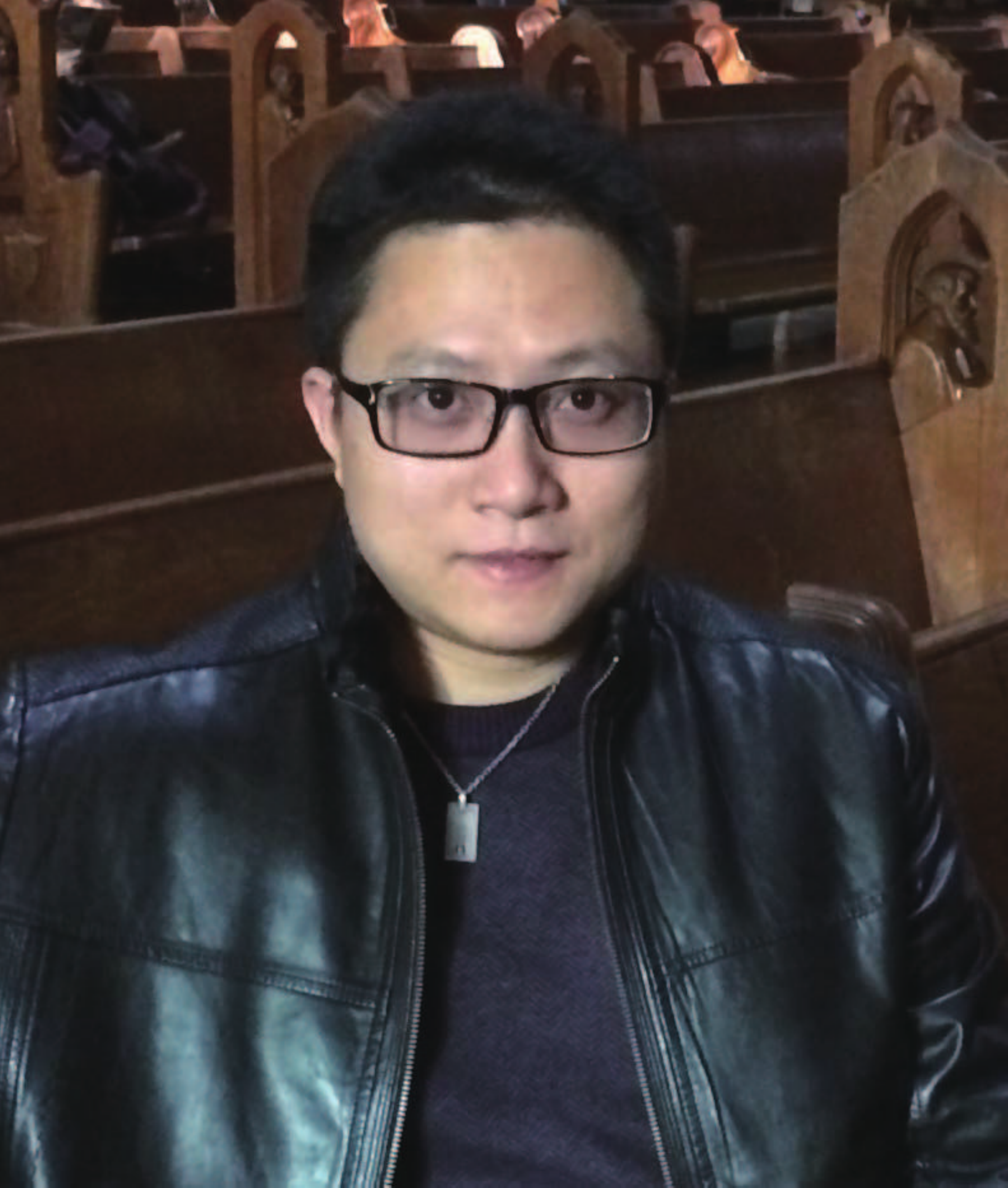}}]{Liang Lin} (M’09, SM’15) is the Executive R\&D Director of SenseTime Group Limited and a Full Professor of Sun Yat-sen University. He is the Excellent Young Scientist of the National Natural Science Foundation of China. He has authorized and co-authorized on more than 100 papers in top-tier academic journals and conferences. He was the recipient of Best Paper Runners-Up Award in ACM NPAR 2010, Google Faculty Award in 2012, Best Paper Diamond Award in IEEE ICME 2017, and Hong Kong Scholars Award in 2014. He is a Fellow of IET.
\end{IEEEbiography}

\begin{IEEEbiography}[{\includegraphics[width=1in,height=1.25in,clip,keepaspectratio]{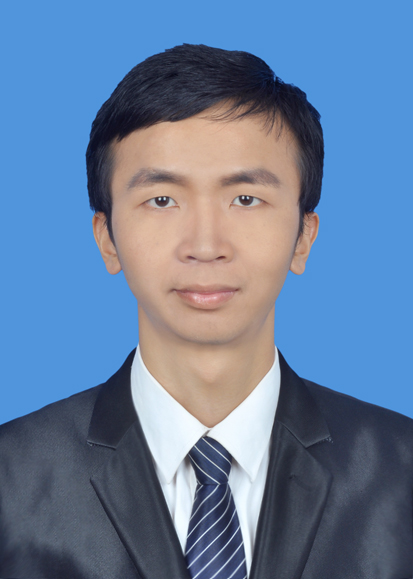}}]{Xiaopeng Yan} received the B.E. degree in Automation at Sun Yat-sen University, Guangzhou, China, in 2017, where he is currently pursuing the master degree  in Computer Science in School of Data and Computer Science, advised by Professor Liang Lin. His current research interests include the application of computer vision (e.g., object detection) and machine learning. 
\end{IEEEbiography}

\begin{IEEEbiography}[{\includegraphics[width=1in,height=1.25in,clip,keepaspectratio]{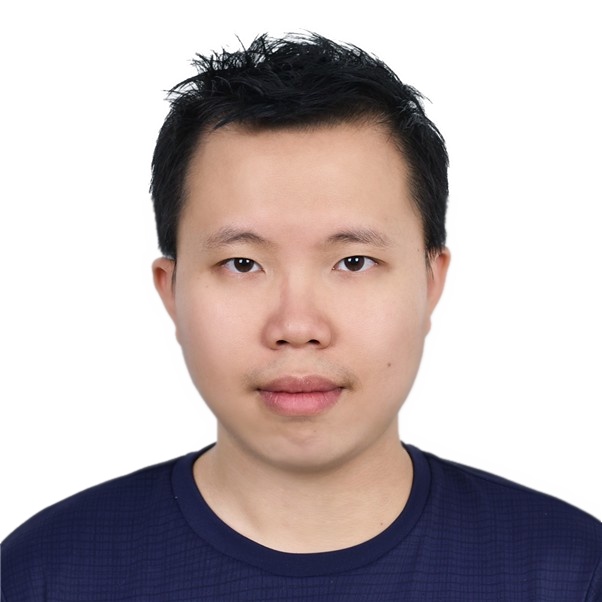}}]{Ziliang Chen} received the BS degree in  Mathematics and Applied Mathematics from Sun Yat-Sen University, Guangzhou, China. He is currently pursuing the Ph.D. degree in computer science and technology at Sun Yat-Sen University, advised by Professor Liang Lin. His current research interests include computer vision and machine learning.
\end{IEEEbiography}

\begin{IEEEbiography}[{\includegraphics[width=1in,height=1.25in,clip,keepaspectratio]{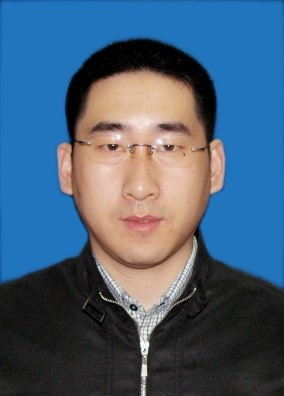}}]{Dongyu Zhang} received the B.S. and Ph.D. degrees from the Harbin Institute of Technology, Harbin,China, in 2003 and 2010, respectively.
He is currently a Research Associate Professor with the School of Data and Computer Science, Sun Yat-sen University, Guangzhou, China. His current research interests include computer vision and machine learning.
\end{IEEEbiography}

\begin{IEEEbiography}[{\includegraphics[width=1in,height=1.25in,clip,keepaspectratio]{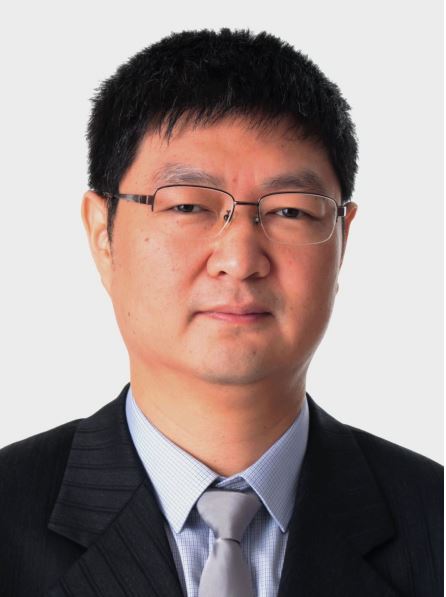}}]{Lei Zhang} (M’04, SM’14, F’18) received his B.Sc. degree in 1995 from Shenyang Institute of Aeronautical Engineering, Shenyang, P.R. China, and M.Sc. and Ph.D degrees in Control Theory and Engineering from Northwestern Polytechnical University, Xi’an, P.R. China, respectively in 1998 and 2001, respectively. From 2001 to 2002, he was a research associate in the Department of Computing, The Hong Kong Polytechnic University. From January 2003 to January 2006 he worked as a Postdoctoral Fellow in the Department of Electrical and Computer Engineering, McMaster University, Canada. In 2006, he joined the Department of Computing, The Hong Kong Polytechnic University, as an Assistant Professor. Since July 2017, he has been a Chair Professor in the same department. His research interests include Computer Vision, Pattern Recognition, Image and Video Analysis, and Biometrics, etc. Prof. Zhang has published more than 200 papers in those areas. As of 2018, his publications have been cited more than 30,000 times in the literature. Prof. Zhang is an Associate Editor of IEEE Trans. on Image Processing, SIAM Journal of Imaging Sciences and Image and Vision Computing, etc. He is a ``Clarivate Analytics Highly Cited Researcher'' from 2015 to 2017. More information can be found in his homepage http://www4.comp.polyu.edu.hk/~cslzhang/.
\end{IEEEbiography}

%

%
%
%




\end{document}